\documentclass[11pt]{article}
\usepackage{arxiv}
\usepackage[utf8]{inputenc}
\usepackage[T1]{fontenc}
\usepackage{graphicx}
\usepackage[numbers,sort&compress]{natbib}
\usepackage{caption}
\usepackage{hyperref}
\hypersetup{colorlinks=true,linkcolor=blue,citecolor=blue,urlcolor=blue}

\usepackage{booktabs}
\usepackage{placeins}

\usepackage{amsmath}
\usepackage{amssymb}
\usepackage{amsthm}

\newtheorem{theorem}{Theorem}
\newtheorem{corollary}{Corollary}
\newtheorem{proposition}{Proposition}
\newtheorem{definition}{Definition}
\newtheorem{lemma}{Lemma}

\title{Stochasticity Is Not the Hard Part: Reduction and Complexity in Instructional Sequencing over Prerequisite DAGs}

\author{
  Zonglin Han\textsuperscript{1} \quad
  Yichen Chen\textsuperscript{1}$^{\dagger}$ \quad
  Jiawen Jiang\textsuperscript{2}$^{\dagger}$ \quad
  Tongan Shi\textsuperscript{3}$^{\dagger}$ \quad
  Kristian A.\ Stevens\textsuperscript{1} \\[6pt]
  \textsuperscript{1}Department of Computer Science, University of California, Davis \\
  \textsuperscript{2}International Digital Economy College, Minjiang University \\
  \textsuperscript{3}School of Computer Science and Artificial Intelligence, Liaoning Normal University \\[4pt]
  {\small $^{\dagger}$Equal contribution (alphabetical order)}
}

\date{}

\begin{document}

\maketitle

\begin{abstract}
When a student must learn concepts connected by prerequisite dependencies, when does the order of instruction matter, and what does it cost to find the best one? We study instructional sequencing as a stochastic shortest-path problem in which attempting a concept succeeds with a state-dependent probability and failure leaves the learner state unchanged. We first prove that this stochasticity can be eliminated \emph{exactly}: the problem collapses to a deterministic shortest-path problem on the lattice of prerequisite order ideals, preserving optimal values and actions. The collapse removes stochastic complexity but not combinatorial complexity: optimal sequencing remains NP-hard---via reduction from feedback arc set in tournaments---even with no prerequisite edges, unit costs, uniform binary nonnegative transfer, and success probabilities at least $1/2$. Hardness is not uniform: when realizable transfer preferences remain jointly acyclic with the prerequisites, any topological order of the residual joint graph is optimal, and fixed prerequisite width yields polynomial-time exact dynamic programming. A computable diagnostic, $m\Delta$, bounds the value of sequencing before optimization. On 70{,}893 interactions from an introductory CS course, the diagnostic certifies a doubly easy regime---little value to optimize and little space to search---while constructed transfer instances realize the challenging regime, where myopic sequencing suffers large regret yet exact A* with a consistent heuristic expands only linearly many states on that family.
\end{abstract}

% ===========================================================================
\section{Introduction}
% ===========================================================================

When a student must learn a set of concepts connected by prerequisite dependencies, what is the computational complexity of finding the learning order that minimizes expected instructional cost? Platforms such as Khan Academy serve more than 180 million learners over prerequisite-structured curricula \citep{khan2025decoder}. Individualized instruction can substantially improve learning outcomes \citep{bloom1984twosigma,vanlehn2011relative,kulik2016effectiveness}, while prerequisite structure can be recovered from curricula and interaction data \citep{stavrinides2023course,lu2019concept}. Yet the computational question remains unresolved: \emph{when does the order of instruction actually matter, and what does it cost to find the best one?}

Existing work largely asks how to generate a learning sequence or policy, using bandits \citep{clement2015multiarmed}, reinforcement learning and partially observable models \citep{rafferty2016faster,chi2011empirically}, prerequisite-aware MDPs \citep{liu2019cseal,cheng2026knowlp}, or, more recently, large language models. Four questions logically precede policy design: whether stochastic learning dynamics genuinely require stochastic planning; when learner-state dependence creates value in sequencing; which transfer structures induce combinatorial conflict; and what governs the difficulty of exact search. We answer all four. The stochastic dynamics can be eliminated \emph{exactly}; the resulting ordering problem remains NP-hard even under severe restrictions; it becomes polynomial when realizable transfer preferences remain jointly acyclic with the prerequisites; and a simple a priori diagnostic bounds the value of sequencing---correctly predicting that the course instances we evaluate lie in a doubly easy regime.

We model a curriculum as a prerequisite learning graph \(G=(V,E)\), a DAG with a learner state given by a downward-closed set of mastered concepts. At state \(s\), the feasible actions are the unmastered concepts whose prerequisites are already contained in \(s\), a graph-theoretic reading of the zone of proximal development \citep{vygotsky1978mind}. Attempting concept \(v\) incurs cost \(T_v\), succeeds with learner-oracle probability \(p(v,s)\), and otherwise leaves the decision state unchanged. Reaching a prerequisite-closed target set at minimum expected total cost defines an instructional sequencing problem over the prerequisite graph, a stochastic shortest-path problem \citep{bertsekas1991analysis}. Because failures preserve the decision state and the success probabilities are time-homogeneous, however, the stochastic self-loops can be eliminated \emph{exactly}: each action becomes a deterministic transition of cost \(T_v/p(v,s)\) on the lattice of prerequisite order ideals, with optimal values and minimizing actions preserved (a complementary deterministic-valuation theory on the same lattice appears in concurrent work \citep{pasechnyuk2026order}). This collapse removes stochastic complexity but not combinatorial complexity. Optimal sequencing remains NP-hard even with no prerequisite edges, unit attempt costs, uniform binary nonnegative transfer, monotone success probabilities, and success probabilities at least \(1/2\).

The hardness is not uniform. For an explicit pairwise-transfer oracle, if the realizable soft ordering preferences and hard prerequisites form an acyclic residual joint graph, then any topological ordering of that graph is globally optimal and can be found in polynomial time; for uniform affine transfer, optimal sequencing is further fixed-parameter tractable in the distance to acyclicity. More generally, if the prerequisite poset has width \(w\), exact dynamic programming is polynomial for fixed \(w\). Between these structural regimes, a computable quantity \(m\Delta\) upper-bounds how much any valid sequence can improve over another, prior to sequence optimization. On 70,893 interactions from an introductory computer science course (ECS32A), this diagnostic certifies a doubly easy regime: the learned evaluator leaves less than \(0.2\%\) model-based improvement available to sequencing, while the target prerequisite closures have width at most four and at most 47 reachable ideals. A larger public curriculum, Junyi Academy, exhibits the opposite topology: sink-target closures reach width 18, with 68 closures exceeding \(10^7\) reachable ideals, so structurally challenging topologies do occur in real curricula. A first-attempt diagnostic proxy further suggests substantially greater state dependence, though the proxy is exploratory rather than a certified bound. Constructed transfer instances realize the challenging regime in full: myopic sequencing incurs large regret exactly where the order-ideal lattice grows exponentially, while exact A* with a consistent heuristic expands only linearly many states on the constructed family.

To our knowledge, this is the first computational characterization of failure-invariant instructional SSPs on prerequisite ideal lattices. We establish an exact stochastic-to-deterministic reduction; NP-hardness alongside residual-acyclic, width-XP, and affine-FPT regimes; the \(m\Delta\) opportunity bound; and exact A* with a consistent heuristic. Experiments contrast doubly easy ECS32A instances with high-regret constructed instances (Sections 3--6).

% ===========================================================================
\section{Related Work}
% ===========================================================================

\subsection{Adaptive Learning and Instructional Sequencing}

Our work characterizes the instructional sequencing problem that adaptive-learning systems attempt to solve. Instructional sequencing has been addressed with expert rules~\citep{vanlehn2011relative,doroudi2019reward}, bandits~\citep{clement2015multiarmed}, reinforcement learning~\citep{chi2011empirically}, POMDP planning~\citep{rafferty2016faster}, and MDP-based learning-path recommendation over prerequisite structures~\citep{liu2019cseal,li2023graph,cheng2026knowlp}. A parallel literature sequences training tasks for artificial agents~\citep{graves2017automated,narvekar2020curriculum}. These methods primarily address how to construct sequences under rich practical constraints, and prior complexity analyses in this literature typically arise from content-selection, resource, or multi-objective formulations~\citep{plp2025complexity}. Our analysis instead isolates a minimal planning core---state-conditioned mastery dynamics on a prerequisite lattice---and derives hardness, tractable subclasses, and instance diagnostics for that core, separating statistical and model-learning challenges from the intrinsic optimization structure.

\subsection{Knowledge Tracing}

Knowledge tracing estimates the learner state our planner conditions on; we contribute no new KT model and remain oracle-agnostic. The lineage runs from Bayesian Knowledge Tracing~\citep{corbett1995knowledge} through recurrent~\citep{piech2015deep} and attention-based architectures~\citep{ghosh2020context}. In our framework a student model enters only through the induced interface $p(v,s)$: the success probability of a feasible concept given a hypothesized mastery state. Models whose native conditioning is an interaction history require an explicit projection onto mastery states, and this projection is part of the modeling choice rather than a free equivalence. For sequencing, what matters is the state variation an oracle induces in collapsed costs---which governs the available planning opportunity---together with its fidelity to the true learner, which governs whether optimizing against it helps.

\subsection{Structure and Search on Ideal Lattices}

The order-ideal state space itself has a long history in mathematical psychology: knowledge space theory models a learner as the subset of items mastered, and prerequisite partial orders induce the quasi-ordinal knowledge spaces represented here as order-ideal lattices~\citep{doignon1985spaces,falmagne2011learning}, with efficient algorithms for assessment and learning-sequence enumeration~\citep{eppstein2013learning}. Our contribution is not the state space but the planning problem on it: expected-cost-optimal sequencing under a state-conditioned stochastic oracle, its exact deterministic collapse, and its complexity.

Our stochastic model is a stochastic shortest-path problem~\citep{bertsekas1991analysis}, and our solver is A* with a consistent heuristic~\citep{hart1968formal}. LAO* extends heuristic search to stochastic shortest-path problems whose solution structures contain loops~\citep{hansen2001lao,nilsson1980principles}; our collapse result shows that such machinery is unnecessary in this subclass, making graph-search A* the natural exact solver, and we retain a LAO*-derived implementation as a solver baseline. \citet{ye2022stochastic} apply SSP to spaced-repetition scheduling of individual items; their formulation optimizes review timing for one memory trace, whereas ours orders many concepts under prerequisite and transfer coupling. Our hardness proof reduces from feedback arc set in tournaments, whose NP-hardness was established by \citet{alon2006ranking} and \citet{charbit2007minimum}; our width-based XP bound rests on Dilworth's decomposition theorem~\citep{dilworth1950decomposition}; and our fixed-parameter result for uniform affine transfer follows from the FPT classification of $\operatorname{MinCSP}(<,\le)$, equivalently Directed Subset Feedback Arc Set~\citep{osipov2024point}.

Closest to our setting, \citet{pasechnyuk2026order} independently studies order sensitivity for deterministic edge-additive interventions on the ideal lattice of a prerequisite poset, characterizing path independence through vanishing diamond curvature and bounding path-value differences by the number of local diamond swaps; the planning consequence is dynamic programming on a depth-truncated Hasse diagram, validated on general sequential-intervention experiments without educational learner-model evaluation. Our setting is complementary: we begin with a stochastic instructional shortest-path problem in which failed learning attempts leave the learner state unchanged, prove an exact reduction to deterministic edge costs $T_v/p(v,s)$, establish worst-case NP-hardness and structural tractable regimes, and develop exact A* search with a consistent heuristic on the implicit lattice. The two order-sensitivity analyses quantify different objects: diamond curvature is a local, second-order mixed difference, whereas our $\Delta_v$ measures the global feasible-state range of a concept's collapsed cost. Local curvature is controlled by global variation but not conversely, and neither resulting global bound uniformly dominates the other; we formalize this relationship in Section~4.

% ===========================================================================
\section{Model and Exact Reduction}
% ===========================================================================

We formalize the instructional sequencing task described above as the \emph{prerequisite sequencing problem}, defined by the following components. The section introduces the primitive instructional SSP, establishes its exact deterministic collapse, records the resulting order-ideal structure, and delineates the boundary of the reduction.

\begin{definition}[Prerequisite structure and mastery states]
Let \(G=(V,E_G)\) be a finite prerequisite DAG given by its transitive reduction. For \(u,v\in V\), write \(u\prec_G v\) when \(G\) contains a directed path from \(u\) to \(v\), and let \(\operatorname{anc}_G(v)=\{u\in V:u\prec_G v\}\) denote the ancestor (full prerequisite) set of \(v\). A mastery state is an order ideal of this partial order:
\[
\mathcal S
=
\bigl\{
s\subseteq V:
v\in s \Longrightarrow \operatorname{anc}_G(v)\subseteq s
\bigr\}.
\]
The initial state \(s_0\in\mathcal S\) is supplied by an upstream assessment procedure. Instruction is confined to a \emph{fixed instructional domain}: a prerequisite-closed set \(Q\) with \(s_0\subseteq Q\subseteq V\)---in our experiments, the prerequisite closure of a designated target set---whose concepts are the only admissible actions, and the goal is \(s^*=Q\). Because transfer can originate outside a target's closure, this exclusion is a modeling assumption rather than a lossless normalization: transfer-inducing detours outside \(Q\) are not represented. We relabel \(Q\) as \(V\) hereafter, so \(s^*=V\) and the complete-mastery case is \(Q=V\).
\end{definition}

Throughout, \(s\) is the complete decision-relevant state: every quantity that affects transition probabilities or costs is a function of \(s\) alone. Mastery is monotone---once a concept crosses the mastery threshold, it remains mastered---matching the absorbing learned-state semantics used in classical knowledge tracing models~\citep{corbett1995knowledge}. This assumption will imply that every successful transition moves upward in the order-ideal lattice.

At a nonterminal state \(s\), the feasible actions are the unmastered concepts on the prerequisite frontier,
\[
\mathcal A(s)
=
\bigl\{
v\in V\setminus s:
\operatorname{anc}_G(v)\subseteq s
\bigr\}.
\]
This frontier gives a graph-theoretic interpretation of the zone of proximal development~\citep{vygotsky1978mind}: instruction is restricted to concepts for which the modeled prerequisites have already been mastered. The frontier is part of the hard-prerequisite semantics of the model rather than an algorithmic pruning rule; soft cross-concept effects among feasible actions are represented through \(p(v,s)\).

\begin{definition}[Primitive instructional SSP]
The primitive instructional SSP is
\[
\mathcal M
=
\bigl(
\mathcal S,s_0,s^*,\mathcal A,P,T
\bigr).
\]
For every \(s\neq s^*\) and \(v\in\mathcal A(s)\), one instructional attempt incurs a finite cost \(T_v>0\) and succeeds with probability \(p(v,s)\in(0,1]\). The transition kernel is
\[
P(s'\mid s,v)
=
\begin{cases}
p(v,s), & s'=s\cup\{v\},\\
1-p(v,s), & s'=s,\\
0, & \text{otherwise}.
\end{cases}
\]
The success probabilities and attempt costs are time-homogeneous. A failed attempt means that the learner has not crossed the mastery threshold, so the complete decision state remains \(s\); a successful attempt adds exactly \(v\). The goal \(s^*\) is absorbing with zero cost. Instruction is serial, so each attempt targets one concept.

For a policy \(\pi\), let \(\tau=\inf\{t:s_t=s^*\}\). The objective is the risk-neutral additive expected total cost
\[
J^\pi(s)
=
\mathbb E^\pi_s
\left[
\sum_{t=0}^{\tau-1} T_{v_t}
\right],
\qquad
J^*(s)=\inf_\pi J^\pi(s).
\]
\end{definition}

Because \(V\) is finite and all feasible success probabilities are positive, \(p_{\min}=\min_{s,\,v\in\mathcal A(s)}p(v,s)>0\). From any state \(s\), every policy needs at most \(|V\setminus s|\) successes, each requiring at most \(1/p_{\min}\) attempts in expectation, so \(\mathbb E^\pi_s[\tau]\le|V\setminus s|/p_{\min}\) and the expected cost is at most \(T_{\max}|V\setminus s|/p_{\min}\), where \(T_{\max}=\max_v T_v\); every policy is proper.

\begin{theorem}[Exact stochastic collapse]
\label{thm:exact-collapse}
Consider the finite primitive instructional SSP above, in which each feasible action has a positive time-homogeneous success probability and finite time-homogeneous attempt cost, failure leaves the complete Markov state unchanged, success has the unique progress successor \(s\cup\{v\}\), and performance is measured by risk-neutral additive expected total cost. Define a deterministic graph on the same state set \(\mathcal S\), with one edge and cost
\[
s\longrightarrow s\cup\{v\},
\qquad
c(s,v)=\frac{T_v}{p(v,s)},
\qquad
v\in\mathcal A(s).
\]
\end{theorem}

Then the primitive SSP and this deterministic shortest-path problem have identical optimal values and identical sets of Bellman-minimizing actions at every state. Every optimal deterministic stationary collapsed policy lifts to an optimal deterministic stationary primitive policy by retrying its selected action after failure until success, and every optimal deterministic stationary primitive policy induces an optimal collapsed policy through the same state-action map.

\paragraph{Proof sketch.}
For a nonterminal state \(s\), the primitive Bellman equation is
\[
J^*(s)
=
\min_{v\in\mathcal A(s)}
\left[
T_v
+
p(v,s)J^*(s\cup\{v\})
+
\bigl(1-p(v,s)\bigr)J^*(s)
\right].
\]
For each feasible action define
\[
\widetilde Q_v(s)
=
\frac{T_v}{p(v,s)}
+
J^*(s\cup\{v\}).
\]
If \(Q_v^{\mathrm{prim}}(s)\) denotes the corresponding primitive Bellman expression, then
\[
Q_v^{\mathrm{prim}}(s)-J^*(s)
=
p(v,s)
\bigl(
\widetilde Q_v(s)-J^*(s)
\bigr).
\]
Since \(p(v,s)>0\), the two differences have the same sign and vanish simultaneously. Hence the primitive and collapsed models have the same minimizing action set, although suboptimal actions need not have the same ranking. It follows that
\[
J^*(s)
=
\min_{v\in\mathcal A(s)}
\left[
\frac{T_v}{p(v,s)}
+
J^*(s\cup\{v\})
\right].
\]
Under a stationary deterministic choice of \(v\), attempts until success are geometric with mean \(1/p(v,s)\), so retrying after each failure realizes exactly the collapsed edge cost and successor, establishing policy lift and recovery. The complete correspondence proof is given in the supplementary material.

\begin{corollary}[Acyclic order-ideal lattice]
\label{cor:collapsed-lattice}
Under monotone mastery, the deterministic graph of Theorem~\ref{thm:exact-collapse} has vertex set \(\mathcal S\) (the order ideals of \(G\)), directed edges \(s\to s\cup\{v\}\) for each \(v\in\mathcal A(s)\), and edge weight \(T_v/p(v,s)\). It is acyclic and layered by mastery cardinality.
\end{corollary}

Every collapsed edge satisfies \(|s\cup\{v\}|=|s|+1\), so no directed cycle is possible. We refer to this deterministic shortest-path problem as the \emph{collapsed Ariadne sequencing problem} and to its state graph as the \emph{collapsed lattice}.

\textbf{Remark.} The self-loop elimination underlying Theorem~\ref{thm:exact-collapse} is standard in time-homogeneous SSPs; the contribution is its rigorous instantiation on the ideal lattice (properness, policy correspondence) and the consequence that LAO*-style loop handling is unnecessary. The reduction generally fails when failures update beliefs, costs depend on attempt history, mastery can be forgotten, actions have multiple outcomes, or the objective is risk-sensitive. It shows the problem is deterministic; it does not show it is easy.

% ===========================================================================
\section{Computational Complexity of the Collapsed Problem}
% ===========================================================================

We now map the computational complexity of the collapsed Ariadne sequencing problem. The landscape ranges from order-insensitivity through polynomial and parameterized tractability to NP-hardness, with an instance-level diagnostic bounding the value of sequencing before sequence optimization.

Fix an initial mastery state \(s_0\), let \(R=V\setminus s_0\), and write \(m=|R|\). Let \(\Sigma_G(s_0)\) denote the prerequisite-valid complete orderings of \(R\). For \(\sigma=(v_1,\ldots,v_m)\), define \(s_i=s_0\cup\{v_1,\ldots,v_i\}\) and
\[
C(\sigma)=\sum_{i=1}^{m} c(v_i,s_{i-1}),
\qquad
c(v,s)=\frac{T_v}{p(v,s)}.
\]
Paths from \(s_0\) to \(s^*\) in the collapsed lattice correspond bijectively to orderings in \(\Sigma_G(s_0)\), so
\[
C^*=\min_{\sigma\in\Sigma_G(s_0)} C(\sigma).
\]

\paragraph{When can sequencing matter?}
For each \(v\in R\), let \(\mathcal F_v=\{s\in\mathcal S:s_0\subseteq s,\ v\in\mathcal A(s)\}\) be its reachable feasible states, and define
\[
\Delta_v
=
\max_{s,s'\in\mathcal F_v}
|c(v,s)-c(v,s')|,
\qquad
\Delta=\max_{v\in R}\Delta_v.
\]
When every \(\Delta_v=0\), each concept has the same collapsed cost whenever it is legally attempted; all valid complete orderings therefore have equal total cost. The following result quantifies the general case.

\begin{proposition}[Sequencing-value bound]
\label{prop:sequencing-value}
For any \(\sigma,\sigma'\in\Sigma_G(s_0)\),
\[
|C(\sigma)-C(\sigma')|
\le
\sum_{v\in R}\Delta_v
\le
m\Delta.
\]
Consequently, \(C(\sigma)-C^*\le m\Delta\) for every valid complete ordering \(\sigma\).
\end{proposition}

\emph{Proof sketch.}
Both orderings master each \(v\in R\) exactly once, possibly at different feasible states. The two contributions of \(v\) differ by at most \(\Delta_v\); summing concept-wise proves the claim. Full proofs appear in the supplementary material. Under monotonicity, the same \(2|R|\) oracle calls also yield a multiplicative bound: \(C(\sigma)/C^*\le\max_v p_v^{\max}/p_v^{\min}\), so large relative regret requires large oracle dynamic range.

Concurrent work bounds order insensitivity via local diamond curvature \(\varepsilon\) and swap distance~\citep{pasechnyuk2026order}; \(\varepsilon\le 2\Delta\) but neither global bound dominates the other, since vanishing curvature can coexist with large \(\Delta_v\). Under a monotone oracle, \(\Delta_v\) is exact in two oracle calls---\(p_{\min}\) at \(s_0\cup\operatorname{anc}_G(v)\), \(p_{\max}\) at \(V\setminus(\{v\}\cup\operatorname{desc}_G(v))\)---giving certified \(\sum_v\Delta_v\) in \(O(|R|)\) calls with no lattice enumeration (supplementary material). For non-monotone oracles the diagnostic may require enumeration or sampling; only certified values yield the formal guarantee.

\paragraph{A polynomial transfer subclass.}
To isolate the role of transfer structure, consider an explicit nonnegative pairwise-transfer oracle
\[
p(v,s)
=
\phi_v\!\left(
b_v^0+
\sum_{u\in s\cap R}w_{uv}
\right),
\qquad
w_{uv}\ge 0,
\]
where \(b_v^0\) is transfer already supplied by \(s_0\) and each \(\phi_v:\mathbb R_{\ge 0}\to(0,1]\) is nondecreasing. Let \(E_H^R=\{(u,v)\in R^2:w_{uv}>0\}\). A soft edge \(u\to v\) is hard-contradictory when \(v\prec_G u\), because every prerequisite-valid sequence must place \(v\) before \(u\). Remove these impossible preferences and define
\[
E_H^+=E_H^R\setminus\{(u,v):v\prec_G u\},
\qquad
D_R^+=\bigl(R,E_G[R]\cup E_H^+\bigr).
\]
For each \(v\), its concept-wise attainable transfer and cost bounds are
\[
\overline A_v
=
b_v^0+
\sum_{u:(u,v)\in E_H^+}w_{uv},
\qquad
\underline c_v
=
\frac{T_v}{\phi_v(\overline A_v)}.
\]

\begin{theorem}[Residual-acyclic optimality]
\label{thm:residual-acyclic}
If \(D_R^+\) is acyclic, every topological ordering of \(D_R^+\) is globally optimal, with
\[
C^*=\sum_{v\in R}\underline c_v.
\]
\end{theorem}

\emph{Proof sketch.}
Any topological ordering of \(D_R^+\) respects the hard prerequisites and places every noncontradictory transfer source \(u\) before its recipient \(v\). Each concept therefore receives \(\overline A_v\) and attains its individual lower bound \(\underline c_v\); their sum is a lower bound for every valid sequence and is attained simultaneously. The full proof and the transfer-edge classification appear in the supplementary material.

With prerequisite reachability available, constructing and checking \(D_R^+\) is linear in \(|R|+|E_G[R]|+|E_H^R|\).

\paragraph{Worst-case hardness.}
Residual acyclicity is a sufficient certificate, not a universal property. To expose the remaining combinatorial difficulty, consider \emph{Restricted Ariadne Sequencing}: the input contains \(n\) concepts, rational attempt costs, a prerequisite DAG, a binary matrix \(B=(b_{uv})\), and a rational threshold \(K\); success probabilities are fixed by
\[
p(v,s)
=
\left(
2-\frac{1}{n}\sum_{u\in s}b_{uv}
\right)^{-1},
\]
and the question is whether a prerequisite-valid complete ordering \(\sigma\) satisfies \(C(\sigma)\le K\).

\begin{theorem}[Hardness of optimal sequencing]
\label{thm:sequencing-hardness}
Restricted Ariadne Sequencing is NP-complete even when the prerequisite graph is empty, all attempt costs are one, transfer is nonnegative with uniform binary magnitude, success probabilities are monotone in \(s\), and every success probability is at least \(1/2\). Consequently, its optimization version is NP-hard under the same restrictions.
\end{theorem}

\emph{Proof sketch.}
Reduce from Minimum Feedback Arc Set in Tournaments~\citep{alon2006ranking,charbit2007minimum}. Given a tournament, set \(b_{uv}=1\) exactly for its arcs and remove all prerequisite edges. Then
\[
c(v,s)
=
2-\frac{1}{n}\sum_{u\in s}b_{uv},
\qquad
C(\sigma)
=
2n-\frac{\operatorname{fwd}(\sigma)}{n}.
\]
Minimizing instructional cost is therefore equivalent to maximizing forward tournament arcs, or minimizing backward arcs. A complete ordering is a polynomial certificate; details appear in the supplementary material. The restrictions show that hardness comes from ordering conflicts induced by state-conditioned transfer, rather than from prerequisite legality, heterogeneous costs, extreme probabilities, negative effects, or stochastic self-loops.

\paragraph{Width-parameterized dynamic programming.}
Worst-case hardness does not preclude exact optimization on narrow prerequisite orders.

\begin{proposition}[Width-based XP algorithm]
\label{prop:width-xp}
Let the prerequisite poset on \(R\) have width \(w\). Then
\[
|\mathcal S_{\mathrm{reach}}|
\le
\left(\frac{m}{w}+1\right)^w,
\]
and exact dynamic programming runs in \(O\bigl(m(m/w+1)^w\bigr)\).

\end{proposition}

\emph{Proof sketch.}
By Dilworth's theorem, the poset partitions into \(w\) chains of lengths \(n_1,\ldots,n_w\) with \(\sum_i n_i=m\)~\citep{dilworth1950decomposition}. Every order ideal is uniquely encoded by one prefix length per chain, so its count is at most \(\prod_i(n_i+1)\), which is bounded by \((m/w+1)^w\) via AM--GM. A reverse-cardinality Bellman pass scans at most \(m\) actions per state. This is XP, not FPT, in \(w\), because the exponent depends on the parameter. In our data, ECS32A closures have \(w\le4\) and at most 47 reachable ideals, whereas Junyi sink closures reach width 18 and 68 exceed our \(10^7\)-ideal enumeration guard, motivating implicit heuristic search. Full derivations appear in the supplementary material.

\paragraph{Distance to residual acyclicity.}
Residual acyclicity is the zero-distance endpoint of a broader parameterized regime. For a valid ordering \(\sigma\), let \(B_H^R(\sigma)\) be the residual soft edges directed backward by \(\sigma\), and define
\[
k
=
\min_{\sigma\in\Sigma_G(s_0)}
|B_H^R(\sigma)\cap E_H^+|.
\]
Equivalently, \(k\) is the minimum number of residual soft edges whose deletion makes \(D_R^+\) acyclic; hard prerequisite edges cannot be deleted.

\begin{theorem}[Affine fixed-parameter tractability]
\label{thm:affine-fpt}
Suppose the collapsed costs have the uniform affine form
\[
c(v,s)
=
a_v
-
\lambda
\bigl|
\{u\in s\cap R:(u,v)\in E_H^+\}
\bigr|,
\qquad
\lambda>0,
\]
with all feasible costs positive. Exact optimal sequencing is fixed-parameter tractable in \(k\): an optimal sequence can be computed in \(f(k)\cdot(|R|+|E_G|+|E_H|)^{O(1)}\) time, where the polynomial is in the encoding of the residual graph, not the lattice.
\end{theorem}

\emph{Proof sketch.}
Each satisfied residual soft edge reduces total cost by exactly \(\lambda\), so
\[
C(\sigma)
=
C_0+
\lambda
|B_H^R(\sigma)\cap E_H^+|.
\]
Cost minimization is therefore exactly minimum residual frustration. This is Directed Subset Feedback Arc Set, equivalently \(\operatorname{MinCSP}(<,\leq)\), with soft precedence constraints breakable and hard prerequisites fixed; known algorithms are FPT in the violated-constraint budget~\citep{osipov2024point}. The supplementary material gives the reduction and sequence recovery. The affine assumption is essential to our reduction: for general nonlinear monotone oracles, fixed-parameter tractability in this distance remains open.

% ===========================================================================
\section{Exact Search on the Implicit Lattice}
% ===========================================================================

The landscape tells us where exact planning is feasible; we now ask how to perform it when the lattice is too large to enumerate.

\paragraph{A consistent sum heuristic.}
For each concept \(v\in R\), let \(\bar p_v\) be any certified upper bound satisfying
\[
\bar p_v
\ge
\max_{s\in\mathcal F_v} p(v,s).
\]
The tightest separable choice is \(p_v^*=\max_{s\in\mathcal F_v}p(v,s)\); under a monotone oracle this is analytic: \(p_v^*=p(v,V\setminus(\{v\}\cup\operatorname{desc}_G(v)))\) (Section~4). When no oracle structure is available, the trivial bound \(\bar p_v=1\) remains valid. Define
\[
h_{\mathrm{sum}}(s)
=
\sum_{v\in V\setminus s}
\frac{T_v}{\bar p_v}.
\]

\begin{proposition}[Consistency of the sum heuristic]
\label{prop:sum-consistency}
The heuristic \(h_{\mathrm{sum}}\) is consistent on the collapsed Ariadne lattice and satisfies \(h_{\mathrm{sum}}(s^*)=0\). It is therefore admissible.
\end{proposition}

\emph{Proof sketch.}
For every collapsed edge \(s\to s\cup\{v\}\),
\[
h_{\mathrm{sum}}(s)-h_{\mathrm{sum}}(s\cup\{v\})
=
\frac{T_v}{\bar p_v}
\le
\frac{T_v}{p(v,s)}
=
c(v,s).
\]
This is exactly the consistency inequality
\(h_{\mathrm{sum}}(s)\le c(v,s)+h_{\mathrm{sum}}(s\cup\{v\})\); goal-zero consistency implies admissibility. A direct concept-wise admissibility proof is given in the supplementary material. Consistency also ensures that graph-search A* never reopens a closed state: when a state is removed from the priority queue, its \(g\)-value is final.

\paragraph{A* on the implicit lattice.}
The order-ideal graph is represented implicitly. Starting from \(s_0\), expanding a state \(s\) enumerates only its frontier actions \(\mathcal A(s)\) and generates successors \(s\cup\{v\}\) on demand. Standard graph-search A* orders open states by
\[
f(s)=g(s)+h_{\mathrm{sum}}(s)
\]
and terminates when the goal is first removed from the queue. Because the heuristic is consistent and all collapsed costs are positive, each state is expanded at most once and the returned path is optimal. In the worst case, A* may still expand every reachable order ideal, reducing to exhaustive dynamic programming with priority-queue overhead; its advantage is that informative lower bounds can avoid generating large portions of a wide lattice.

\paragraph{Solver comparison.}
We compare full reverse-topological dynamic programming, Dijkstra's algorithm (\(h=0\)), A* with \(h_{\mathrm{sum}}\), and a LAO*-derived solution-graph baseline. All four return identical optimal values to within \(10^{-12}\). A* provides the strongest expansion and wall-clock performance on the wide instances, whereas the additional revision machinery of the LAO*-derived solver offers no benefit after collapse---consistent with Theorem~\ref{thm:exact-collapse}, which leaves no stochastic loops for that machinery to exploit. Detailed counts appear in Section~6 and the supplementary material.

The structural analysis motivates two empirical regimes; we now test where real and constructed instances land.

% ===========================================================================
\section{Experiments}
% ===========================================================================

\paragraph{Setup.}
We evaluate Ariadne on 70,893 mapped interactions from 294 students in an introductory computer science course (ECS32A), organized by a 61-concept prerequisite DAG. The evaluation uses ten prespecified target closures, an empty initial mastery state, and a uniform nominal attempt cost \(T_v=60\), so expected cost is proportional to expected attempts. A deterministic CPU \emph{FrozenMonotonicOracle} serves as both Ariadne's planning oracle and the public evaluator for every condition; on held-out sessions it is only moderately predictive (full-feature AUC \(0.775\); zero-history ablation AUC \(0.611\); no estimate over prerequisite-closed planner states is available), so our conclusions concern the planning landscape induced by this frozen model rather than true learner counterfactuals. The closed-loop design makes the exact Ariadne reference zero-regret by construction; the informative comparisons are the deviations of the remaining methods from that evaluator-defined optimum. Normalized regret of a sequence \(\sigma\) is \(\bigl(C_{\mathrm{eval}}(\sigma)-J^{*}_{\mathrm{eval}}\bigr)/J^{*}_{\mathrm{eval}}\), where \(C_{\mathrm{eval}}\) is its cost and \(J^{*}_{\mathrm{eval}}\) the exact optimum under the frozen evaluator for that target; repetitions are averaged within each target (100 seeds for Random Frontier, 20 per LLM condition, one run otherwise) before the ten targets are weighted equally. Table~\ref{tab:main-results} reports the ten primary conditions: Ariadne, Frequency, BKT-set, and DKT-set planners under Exact and Greedy solvers, plus Random Frontier and Linear Syllabus orderings; Exact rows are computed by graph-search A*, and the LAO*-derived implementation returns identical values (Section~5). Four tool-free LLM conditions evaluated under the identical protocol land in the same easy band---best mean regret \(6.6\times10^{-5}\) (GPT-5.6 SOL, full context), with full context outperforming zero context for both models (Holm-adjusted exact permutation \(p=0.016\))---and are analyzed, with validity audits, in the supplementary material, together with full data processing, oracle validation, and provenance.

\begin{table}[t]
\centering
\small
\caption{Sequence quality under the frozen public evaluator (repetitions averaged within each target; ten targets weighted equally). Ariadne Exact shares the evaluator's planning oracle by design, so its zero regret is a protocol identity rather than an empirical claim about real-student optimality. Mean costs for all conditions appear in the supplementary material.}
\label{tab:main-results}
\begin{tabular}{lr}
\toprule
Condition & Normalized regret \\
\midrule
Ariadne Exact & \(0\) \\
Ariadne Greedy & \(5.934\times10^{-7}\) \\
Frequency Greedy & \(1.047\times10^{-5}\) \\
DKT-set Greedy & \(4.935\times10^{-5}\) \\
DKT-set Exact & \(1.828\times10^{-4}\) \\
BKT-set Greedy & \(1.995\times10^{-4}\) \\
BKT-set Exact & \(2.481\times10^{-4}\) \\
Random Frontier & \(2.633\times10^{-4}\) \\
Frequency Exact & \(2.658\times10^{-4}\) \\
Linear Syllabus & \(3.385\times10^{-4}\) \\
\bottomrule
\end{tabular}
\end{table}

\paragraph{Q1: Does the diagnostic predict sequencing opportunity?}
For the FrozenMonotonicOracle, the target-level opportunity bound satisfies \(\sum_v\Delta_v/J^*<0.2\%\); the BKT- and DKT-derived set oracles expose wider envelopes of approximately \(1.9\%\)--\(4.2\%\) and \(2.1\%\)--\(4.7\%\), respectively. Table~\ref{tab:main-results} confirms that the largest target-equal mean regret is only \(3.385\times10^{-4}\) (\(0.034\%\)), about one sixth of the Frozen-oracle envelope. The table also reveals oracle mismatch. BKT-set and DKT-set Exact match reverse-topological dynamic programming under their own objectives and improve or tie their corresponding Greedy policies internally, yet score worse under the Frozen evaluator. Exact optimization can therefore amplify learner-model error; this reversal is not a solver-correctness failure.

\paragraph{Q2: Why are the ECS32A instances easy?}
The objective leaves little value to optimize, and the prerequisite topology leaves little state space to search. Across the ten closures, width is at most four and the number of reachable ideals is at most 47, making exhaustive dynamic programming effectively instantaneous. This is a double degeneracy: weak feasible-state cost variation makes nearly every valid order competitive, while narrow prerequisite structure makes exact enumeration cheap.

\paragraph{A wider public topology.}
The public Junyi Academy corpus contains 835 exercises and exhibits the opposite structural regime (topology figure in the supplementary material). Its sink-target closures have median width eight and maximum width 18; their median ideal count is 96,608, and 68 closures exceed the \(10^7\)-state enumeration guard. A first-attempt bucketing proxy gives median proxy ratio \(\sum_v\Delta_v/J^*_{\mathrm{proxy}}=5.6\%\) and a maximum of \(12.5\%\), suggesting substantially greater state dependence than in ECS32A. This proxy is not directly comparable to the trained Frozen-oracle diagnostic: students who have already mastered prerequisites may be systematically stronger, inflating apparent state dependence. The proxy is computed only on the observable restriction of each closure (241 closures yield values; median concept coverage \(0.78\)), so it is an exploratory diagnostic rather than a certified full-closure bound. We therefore use Junyi only as topology and diagnostic evidence, not as planner evaluation.

\paragraph{Q3: When does the heuristic reduce exact search?}
On the narrow ECS32A closures, A* and Dijkstra expand the same states. This is expected: the intermediate analytic certificate \(\bar p_v=p(v,V\setminus\{v\})\) used in our experiments leaves a \(2.5\%\)--\(3.2\%\) residual gap at \(s_0\), which exceeds the total cost spread \(\sum_v\Delta_v/J^*<0.2\%\) by an order of magnitude, so \(f\)-values cannot separate optimal from suboptimal states regardless of lattice size. On wider synthetic lattices where \(\sum_v\Delta_v\) is large relative to the certificate slack, the distinction appears. Relative to Dijkstra, \(h_{\mathrm{sum}}\) reduces expansions by about \(2\times\) at width two, \(4\times\) at width three, and \(10\)--\(25\times\) at width five. Replacing the trivial certificate \(\bar p_v=1\) with this analytic bound further tightens search without changing the optimal value; complete counts are reported in the supplementary material.

\paragraph{Q4: Can myopic sequencing be substantially suboptimal?}
For each width \(w\in\{3,5,7,10,15\}\), the sibling-transfer trap has a root \(r\), an antichain of \(w\) siblings \(a_1,\ldots,a_w\), and a sink \(z\), with prerequisite edges \(r\to a_i\to z\). Success probabilities are \(p(v,s)=\sigma(\varphi_v+1.5\sum_u w_{uv}\mathbf{1}[u\in s])\) with \(T_v=60\), \(\varphi_r=1\), \(\varphi_z=0\), \(\varphi_{a_i}=-0.8+0.12i\); transfer weights form a non-prerequisite chain \(w_{a_i,a_{i+1}}=1\). Greedy favors later siblings (higher baseline \(p\)) but the forward order \(a_1,\ldots,a_w\) exploits sequential transfer. The poset width is exactly \(w\), giving \(2^w+2\) reachable ideals. Figure~\ref{fig:trap} shows that DP expands \(2^w+1\) states while A* expands only \(w+3\); greedy regret ranges from \(28.3\%\) to \(45.1\%\). The \(O(w)\) A* profile depends on \(h_{\mathrm{sum}}\) being near-tight on this family and is not a general complexity claim. By contrast, the NP-hardness family of Theorem~3 has \(p\ge1/2\), forcing \(c(v,s)\in[1,2]\) and a multiplicative ratio at most two; hardness there concerns exact optimization, not the magnitude of achievable gains.

\begin{figure}[t]
\centering
\includegraphics[width=0.65\textwidth]{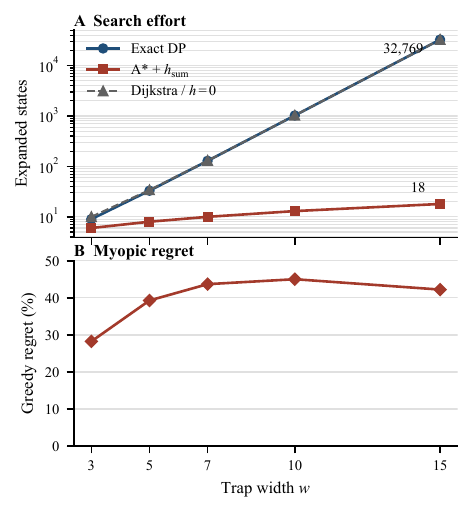}
\caption{Sibling-transfer trap family. Exact DP expands all \(2^w+1\) ideals and Dijkstra nearly as many, whereas A* with \(h_{\mathrm{sum}}\) expands \(w+3\); greedy regret remains \(28.3\%\)--\(45.1\%\). These scalings are family-specific.}
\label{fig:trap}
\end{figure}

All splits, seeds, model checkpoints, LLM prompts and responses, diagnostic artifacts, and aggregation scripts are frozen and hash-addressed; code accompanies the submission.

% ===========================================================================
\section{Conclusion}
% ===========================================================================

Under failure-invariant, time-homogeneous learner dynamics, instructional sequencing collapses exactly to deterministic shortest paths on the order-ideal lattice, yet remains NP-hard in general; residual acyclicity, bounded width, and---for uniform affine transfer---small distance to acyclicity identify tractable regimes. The \(m\Delta\) diagnostic bounds the value of sequencing before optimization, and a consistent heuristic supports exact A* search on the implicit lattice.

Our empirical conclusions are model-relative: the closed-loop ECS32A evaluation makes exact-planner zero regret a protocol identity, not evidence of real-student optimality, and the Junyi analysis provides topology evidence without planner evaluation. The residual-acyclic and FPT results assume explicit nonnegative pairwise transfer (affine costs for the latter) and do not extend automatically to black-box oracles.

Extensions include belief-space planning for failure-dependent updates, nonmonotone state models for forgetting, validated oracles for Junyi-scale evaluation, and robust learner modeling to address misspecification. Planning is a lever; the learner model is its fulcrum.

% ===========================================================================

% ===========================================================================
% SUPPLEMENTARY MATERIAL (appended for arXiv)
% ===========================================================================

\appendix
\newpage

\section*{Supplementary Material}
\addcontentsline{toc}{section}{Supplementary Material}

The following appendices contain complete proofs of all formal results stated
in the main text, detailed experimental protocols, full result tables,
and diagnostic analyses.

\bigskip

\section{Proof of Theorem 1: Exact Stochastic Collapse}
\label{supp:exact-collapse}

\subsection{Setup and Conventions}
\label{supp:collapse-setup}

We use the model and notation of Definitions~1--2 in the main paper. In particular, \(G=(V,E_G)\) is a finite prerequisite DAG, \(\mathcal S\) is its set of mastery order ideals, \(s_0\in\mathcal S\) is the initial mastery state, \(s^*=V\) is the goal, and
\[
\mathcal A(s)
=
\{v\in V\setminus s:\operatorname{anc}_G(v)\subseteq s\}
\]
is the prerequisite frontier. If the original request specifies a prerequisite-closed target domain \(Q\), we restrict the model to \(Q\) under the fixed-domain assumption \(s_0\subseteq Q\) of Definition~1 and relabel that domain as \(V=Q\). All statements below are made on this fixed domain.

We distinguish direct parents from strict ancestors. Let
\[
\operatorname{par}_G(v)=\{u:(u,v)\in E_G\},
\qquad
\operatorname{anc}_G(v)=\{u:u\prec_G v\}.
\]
For any \(s\subseteq V\), closure under direct parents implies closure under all ancestors: if \(u=u_0\to u_1\to\cdots\to u_\ell=v\) is a directed path and \(v\in s\), repeated application of parent closure gives \(u_{\ell-1},\ldots,u_0\in s\). The converse is immediate because every direct parent is an ancestor. Hence
\[
s\text{ is parent-closed}
\quad\Longleftrightarrow\quad
s\text{ is ancestor-closed}.
\]
Moreover, if \(s\) is an order ideal and \(v\notin s\), then
\[
\operatorname{par}_G(v)\subseteq s
\quad\Longleftrightarrow\quad
\operatorname{anc}_G(v)\subseteq s.
\]
Indeed, the reverse implication is immediate; for the forward implication, every strict ancestor of \(v\) is either a direct parent or an ancestor of a direct parent, and parent closure of \(s\) completes the argument. Thus the direct-parent and ancestor-based frontier tests are equivalent on \(\mathcal S\).

A policy may be deterministic or randomized and may depend on the full history, but it is always nonanticipative. Together with the transition kernel, a policy \(\pi\) induces a probability measure \(\mathbb P^\pi_s\) on trajectories starting from \(s\), with natural filtration
\[
\mathcal F_t
=
\sigma(s_0,a_0,s_1,a_1,\ldots,s_t).
\]
The quantities \(p(v,s)\) and \(T_v\) are time-homogeneous and depend only on the current state-action pair.

\subsection{Well-Posedness and Properness}
\label{supp:properness}

\begin{lemma}[Well-posedness and properness]
\label{lem:supp-properness}
The primitive instructional SSP is finite and well-defined. Every nonterminal state has at least one feasible action, every transition remains in \(\mathcal S\), every nonterminal action has positive finite cost, and \(s^*=V\) is absorbing with zero cost. Moreover, every nonanticipative policy that selects only feasible actions is proper from every state and has finite expected total cost. Specifically, with
\[
p_{\min}
=
\min_{\substack{s\in\mathcal S,\ s\neq s^*\\v\in\mathcal A(s)}}p(v,s)>0,
\qquad
T_{\max}=\max_{v\in V}T_v<\infty,
\]
one has, for every \(s\in\mathcal S\),
\[
\mathbb E^\pi_s[N_{\mathrm{att}}]
\le
\frac{n-|s|}{p_{\min}},
\qquad
J^\pi(s)
\le
T_{\max}\frac{n-|s|}{p_{\min}},
\]
where \(n=|V|\) and \(N_{\mathrm{att}}\) is the number of nonterminal attempts before reaching \(s^*\).
\end{lemma}

\begin{proof}
Because \(V\) is finite,
\[
|\mathcal S|\le 2^n,
\]
and every action set is finite. Fix a non-goal state \(s\in\mathcal S\). Since \(V\setminus s\neq\varnothing\), choose a \(\prec_G\)-minimal element \(v\) of \(V\setminus s\). If some ancestor \(u\prec_G v\) were also outside \(s\), then \(u\) would contradict the minimality of \(v\). Hence \(\operatorname{anc}_G(v)\subseteq s\), so \(v\in\mathcal A(s)\). Therefore,
\[
\mathcal A(s)\neq\varnothing
\qquad
\text{for every }s\neq s^*.
\]

For transition closure, let \(v\in\mathcal A(s)\). A failure returns to \(s\in\mathcal S\). For a success, consider \(s^+=s\cup\{v\}\). If \(x\in s\), then \(\operatorname{anc}_G(x)\subseteq s\subseteq s^+\) because \(s\) is an ideal. If \(x=v\), then \(\operatorname{anc}_G(v)\subseteq s\subseteq s^+\) by feasibility. Thus \(s^+\in\mathcal S\). Since \(p(v,s)\in(0,1]\), the two outcomes define a valid probability distribution on \(\mathcal S\).

Every nonterminal action has cost \(T_v\in(0,\infty)\), and the model assigns a zero-cost absorbing action to \(s^*=V\). This proves the basic finite-state and kernel claims.

It remains to prove properness uniformly over arbitrary nonanticipative feasible policies. The set of feasible state-action pairs is finite, and every corresponding probability is strictly positive, so \(p_{\min}>0\). Consider a phase beginning immediately after some success, or at the initial state, and ending at the next success. The policy may change actions after failures. Let \(F_k\) denote the event that the first \(k\) attempts in the phase all fail, with \(F_0\) the sure event, and let \(\tau_k\) be the time of the \(k\)-th attempt in that phase. Conditional on \(\mathcal F_{\tau_k}\) and on \(F_{k-1}\), the selected action is feasible, so its failure probability is at most \(1-p_{\min}\). Therefore,
\[
\begin{aligned}
\mathbb P^\pi_s(F_k)
&=
\mathbb E^\pi_s
\left[
\mathbf 1_{F_{k-1}}
\mathbb P^\pi_s
\left(
\text{attempt \(k\) fails}\mid \mathcal F_{\tau_k}
\right)
\right] \\
&\le
(1-p_{\min})\mathbb P^\pi_s(F_{k-1})
\le
(1-p_{\min})^k .
\end{aligned}
\]
Hence the probability of failing forever in the phase is zero. If \(W\) is the number of attempts in that phase, the tail-sum formula gives
\[
\mathbb E^\pi_s[W]
=
\sum_{k\ge 0}\mathbb P^\pi_s(W>k)
\le
\sum_{k\ge 0}(1-p_{\min})^k
=
\frac{1}{p_{\min}}.
\]
Each success adds exactly one previously unmastered concept, and mastery is never lost. Starting from \(s\), at most \(n-|s|\) successful phases are therefore required to reach \(V\). Since every phase terminates almost surely, the goal is reached almost surely under \(\pi\). Thus every feasible policy is proper.

No independence between phases is needed for the expectation bound. Applying the conditional phase bound successively and using linearity of expectation yields
\[
\mathbb E^\pi_s[N_{\mathrm{att}}]
\le
\frac{n-|s|}{p_{\min}}.
\]
Every attempt costs at most \(T_{\max}\), so
\[
J^\pi(s)
=
\mathbb E^\pi_s
\left[
\sum_{t=0}^{\tau-1}T_{a_t}
\right]
\le
T_{\max}\mathbb E^\pi_s[N_{\mathrm{att}}]
\le
T_{\max}\frac{n-|s|}{p_{\min}}
<
\infty.
\]
\end{proof}

\subsection{Existence of an Optimal Deterministic Stationary Policy}
\label{supp:stationary-existence}

The primitive model is a finite stochastic shortest-path problem with an absorbing zero-cost goal and nonnegative stage costs. Standard finite-SSP results state that if at least one proper policy exists and every improper policy has infinite expected cost from at least one state, then the Bellman optimality equation has the unique solution \(J^*\), and any deterministic stationary policy greedy with respect to \(J^*\) is optimal~\citep{bertsekas1991analysis}. Lemma~\ref{lem:supp-properness} proves the stronger fact that every feasible policy is proper and has finite expected cost. The improper-policy condition is therefore vacuous, and the standard result applies directly.

Consequently, \(J^*\) is the unique solution of
\[
J^*(s)
=
\min_{v\in\mathcal A(s)}
\left[
T_v
+
p(v,s)J^*(s\cup\{v\})
+
(1-p(v,s))J^*(s)
\right],
\qquad
J^*(s^*)=0,
\]
and there exists an optimal deterministic stationary policy. This fact is used below only to convert statewise Bellman minimizers into globally optimal primitive policies.

\subsection{Exact Collapse}
\label{supp:collapse-proof}

\subsubsection{Per-action algebra and preservation of minimizers}

Fix a nonterminal state \(s\). Finiteness of \(J^*(s)\), established in Lemma~\ref{lem:supp-properness}, is essential for the following rearrangement. For each \(v\in\mathcal A(s)\), define the primitive Bellman expression
\[
Q^{\mathrm{prim}}(s,v)
=
T_v
+
p(v,s)J^*(s\cup\{v\})
+
(1-p(v,s))J^*(s)
\]
and the collapsed expression
\[
\widetilde Q(s,v)
=
\frac{T_v}{p(v,s)}
+
J^*(s\cup\{v\}).
\]
Then
\begin{equation}
\label{eq:supp-argmin-identity}
Q^{\mathrm{prim}}(s,v)-J^*(s)
=
p(v,s)
\left[
\widetilde Q(s,v)-J^*(s)
\right].
\end{equation}
Because \(p(v,s)>0\), the two sides of \eqref{eq:supp-argmin-identity} have the same sign and vanish simultaneously. The Bellman equation implies
\[
Q^{\mathrm{prim}}(s,v)\ge J^*(s)
\qquad
\text{for every }v\in\mathcal A(s),
\]
with equality exactly for primitive Bellman minimizers. Equation~\eqref{eq:supp-argmin-identity} therefore gives
\[
\arg\min_{v\in\mathcal A(s)}
Q^{\mathrm{prim}}(s,v)
=
\arg\min_{v\in\mathcal A(s)}
\widetilde Q(s,v),
\]
and at least one collapsed expression equals \(J^*(s)\). Hence
\begin{equation}
\label{eq:supp-collapsed-bellman}
J^*(s)
=
\min_{v\in\mathcal A(s)}
\left[
\frac{T_v}{p(v,s)}
+
J^*(s\cup\{v\})
\right].
\end{equation}
Only the minimizing sets are preserved; suboptimal actions may be ranked differently because the positive factor \(p(v,s)\) in \eqref{eq:supp-argmin-identity} depends on \(v\).

For intuition, if a stationary policy selects \(v\) at \(s\), every failure returns to the same state and the policy selects \(v\) again. The number of attempts until success is geometric with mean \(1/p(v,s)\), so the expected accumulated attempt cost before the first successful transition is exactly \(T_v/p(v,s)\).

\subsubsection{Equivalence to shortest paths on the order-ideal lattice}

Let \(L\) be the directed graph with vertex set \(\mathcal S\) and, for each \(s\neq s^*\) and \(v\in\mathcal A(s)\), an edge
\[
s\longrightarrow s\cup\{v\}
\qquad
\text{of weight}
\qquad
c(v,s)=\frac{T_v}{p(v,s)}.
\]
The graph is finite. Every edge strictly increases mastery cardinality:
\[
|s\cup\{v\}|=|s|+1.
\]
Thus \(L\) is acyclic and layered by \(|s|\).

Lemma~\ref{lem:supp-properness} shows that every nonterminal state has a frontier action; repeated strict-progress transitions therefore reach \(s^*\) from every \(s\), so the following distances are finite. Let \(D(s)\) denote the shortest-path distance from \(s\) to \(s^*\) in \(L\). We prove \(D(s)=J^*(s)\) by reverse induction on
\[
r(s)=n-|s|.
\]
If \(r(s)=0\), then \(s=s^*\) and
\[
D(s^*)=0=J^*(s^*).
\]
Assume the equality holds for every state \(q\) with \(r(q)<k\), and let \(r(s)=k>0\). Every successor \(s\cup\{v\}\) has rank \(k-1\), so the induction hypothesis and \eqref{eq:supp-collapsed-bellman} give
\[
\begin{aligned}
D(s)
&=
\min_{v\in\mathcal A(s)}
\left[
\frac{T_v}{p(v,s)}
+
D(s\cup\{v\})
\right] \\
&=
\min_{v\in\mathcal A(s)}
\left[
\frac{T_v}{p(v,s)}
+
J^*(s\cup\{v\})
\right]
=
J^*(s).
\end{aligned}
\]
Therefore the primitive optimal value and collapsed shortest-path distance agree at every state. The same reverse-cardinality induction is the reverse-topological dynamic-programming solver used as the exact enumeration baseline in the main paper.

\subsubsection{Deterministic stationary policy correspondence}

We first lift a collapsed policy to the primitive SSP. Let \(\mu:\mathcal S\setminus\{s^*\}\to V\) be a deterministic stationary policy that is optimal for the collapsed shortest-path problem from every state. Then
\[
\mu(s)
\in
\arg\min_{v\in\mathcal A(s)}
\widetilde Q(s,v)
\qquad
\text{for every }s\neq s^*.
\]
Define the primitive deterministic stationary policy \(\pi_\mu\) by
\[
\pi_\mu(s)=\mu(s).
\]
If an attempt fails, the state remains \(s\); stationarity therefore makes \(\pi_\mu\) select the same action again, implementing retry-until-success automatically. By \eqref{eq:supp-argmin-identity}, \(\mu(s)\) is also a primitive Bellman minimizer at every state. The SSP optimality result of the preceding subsection then implies that \(\pi_\mu\) is primitive-optimal.

Conversely, let \(\pi\) be a deterministic stationary primitive policy that is optimal from every state. Then \(J^\pi=J^*\) pointwise, and its policy-evaluation equation gives
\[
J^*(s)
=
T_{\pi(s)}
+
p(\pi(s),s)J^*(s\cup\{\pi(s)\})
+
(1-p(\pi(s),s))J^*(s)
=
Q^{\mathrm{prim}}(s,\pi(s)).
\]
Hence \(\pi(s)\) is a primitive Bellman minimizer. Equation~\eqref{eq:supp-argmin-identity} implies that the same action is a collapsed Bellman minimizer, so the identical state-action map is an optimal deterministic stationary policy on \(L\).

The scope of this correspondence is deliberate. We do not claim a bijection between all optimal primitive policies and individual collapsed paths. If several actions tie, a randomized policy may mix Bellman minimizers, and a history-dependent policy may switch among minimizers after failures even though the mastery state is unchanged. Such policies may remain optimal but need not determine a single collapsed path. Likewise, a policy known to be optimal only from \(s_0\) is constrained only on states reachable from \(s_0\) under that policy; arbitrary choices on unreachable states do not affect its start-state value. The pointwise correspondence above concerns deterministic stationary policies that are optimal from every state, while the start-state version holds on the relevant reachable subgraph.

\subsubsection{Lattice paths and feasible orderings}

Let \(R=V\setminus s_0\), and let \(\Sigma_G(s_0)\) be the set of complete prerequisite-valid orderings
\[
\sigma=(v_1,\ldots,v_m),
\qquad
m=|R|,
\]
such that, with \(s_i=s_0\cup\{v_1,\ldots,v_i\}\),
\[
v_i\in\mathcal A(s_{i-1})
\qquad
\text{for every }i=1,\ldots,m.
\]

Every path in \(L\) from \(s_0\) to \(V\) adds exactly one new concept per edge. Reading those concepts in edge order yields a sequence containing each element of \(R\) exactly once. Edge feasibility gives \(v_i\in\mathcal A(s_{i-1})\), so the sequence belongs to \(\Sigma_G(s_0)\).

Conversely, if \(\sigma\in\Sigma_G(s_0)\), every prefix state \(s_i\) is an order ideal: \(s_0\) is an ideal, and adding a frontier concept preserves ideal closure. Therefore each transition
\[
s_{i-1}\longrightarrow s_i=s_{i-1}\cup\{v_i\}
\]
is an edge of \(L\), producing a path from \(s_0\) to \(V\). The two constructions are inverse, so lattice paths and complete feasible orderings are in bijection. Under this correspondence, path weight equals ordering cost:
\[
C(\sigma)
=
\sum_{i=1}^{m}
\frac{T_{v_i}}{p(v_i,s_{i-1})}.
\]
This is the interface used by the ordering formulation and complexity results in Section~4 of the main paper.

\section{Proof of Proposition 1: Sequencing-Value Bound}
\label{sec:supp-prop1}

Proposition~1 of the main paper bounds the maximum cost difference between any two prerequisite-valid complete sequences in terms of the per-concept state-dependence magnitudes \(\Delta_v\).

Recall from the main text: for each remaining concept \(v\in R=V\setminus s_0\),
\[
\Delta_v
=
\max_{\substack{s,\,s'\in F_v}}
\bigl|c(v,s)-c(v,s')\bigr|,
\]
where \(F_v=\{s\in\mathcal S_{\mathrm{reach}}:v\in\mathcal A(s)\}\) is the set of reachable states in which \(v\) is feasible, and \(\Delta=\max_{v\in R}\Delta_v\).

\begin{proposition}[Sequencing-value bound; Proposition~1 of the main paper]
\label{prop:supp-mdelta}
For any two prerequisite-valid complete sequences \(\sigma,\sigma'\in\Sigma_G(s_0)\),
\[
|C(\sigma)-C(\sigma')|
\;\le\;
\sum_{v\in R}\Delta_v
\;\le\;
m\,\Delta,
\]
where \(m=|R|\). In particular, the regret of any sequence relative to the optimum is at most \(m\Delta\).
\end{proposition}

\begin{proof}
Both sequences master the same set \(R\), so
\[
C(\sigma)-C(\sigma')
=
\sum_{v\in R}
\bigl[c(v,s_\sigma(v))-c(v,s_{\sigma'}(v))\bigr],
\]
where \(s_\sigma(v)\) denotes the mastery state at which \(v\) is attempted under \(\sigma\). For each \(v\), both \(s_\sigma(v)\) and \(s_{\sigma'}(v)\) are reachable states in which \(v\) is feasible (they lie in \(F_v\)), so by definition of \(\Delta_v\),
\[
\bigl|c(v,s_\sigma(v))-c(v,s_{\sigma'}(v))\bigr|\le\Delta_v.
\]
Taking absolute values and applying the triangle inequality,
\[
|C(\sigma)-C(\sigma')|
\le
\sum_{v\in R}\Delta_v
\le
m\cdot\max_{v\in R}\Delta_v
=
m\,\Delta.
\qedhere
\]
\end{proof}

Two boundary cases anchor the bound. When \(\Delta=0\), every concept's collapsed cost is invariant across all reachable feasible states, and the bound gives \(|C(\sigma)-C(\sigma')|=0\), recovering the equal-cost conclusion of the state-independent corollary. At the opposite extreme, the NP-hardness construction of Section~\ref{sec:supp-thm3} has \(\Delta_v>0\) for every concept, and the exact cost identity \(C(\sigma)=2n-\mathrm{fwd}(\sigma)/n\) shows that the \(\sum_v\Delta_v\) bound is tight to within a factor of two on that family.

The tighter bound \(\sum_v\Delta_v\) should be preferred when individual \(\Delta_v\) values are available; \(m\Delta\) is a convenience corollary that requires only the global maximum. In the main-text experiments, the reported diagnostic uses the tighter \(\sum_v\Delta_v\) form, not the coarser \(m\cdot\max_v\Delta_v\).

\subsection{Analytic Computation under Monotone Oracles}

When the oracle is monotone (\(s\subseteq s'\Rightarrow p(v,s)\le p(v,s')\) for all feasible \(v\)), the extremal feasible states for each concept can be identified without lattice enumeration.

\begin{lemma}[Extremal feasible states]
\label{lem:supp-extremal-states}
Let the oracle be monotone. For every \(v\in R\):
\begin{enumerate}
\item The state \(s_{\min}(v)=s_0\cup\operatorname{anc}_G(v)\) is the \(\subseteq\)-smallest element of \(F_v\).
\item The state \(s_{\max}(v)=V\setminus(\{v\}\cup\operatorname{desc}_G(v))\) is the \(\subseteq\)-largest element of \(F_v\).
\end{enumerate}
\end{lemma}

\begin{proof}
\emph{Part 1.} The set \(s_0\cup\operatorname{anc}_G(v)\) is an order ideal: \(s_0\) is an ideal by assumption, \(\operatorname{anc}_G(v)\) is downward-closed by transitivity, and the union of two ideals is an ideal. It is reachable from \(s_0\) (add the elements of \(\operatorname{anc}_G(v)\setminus s_0\) in any topological order). Since \(\operatorname{anc}_G(v)\subseteq s_0\cup\operatorname{anc}_G(v)\) and \(v\notin s_0\cup\operatorname{anc}_G(v)\), we have \(v\in\mathcal A(s_{\min}(v))\), so \(s_{\min}(v)\in F_v\). Now let \(s\in F_v\) be arbitrary. Because \(s\) is an ideal with \(v\in\mathcal A(s)\), all ancestors of \(v\) lie in \(s\); together with \(s_0\subseteq s\) (every reachable state contains \(s_0\)), we get \(s_{\min}(v)\subseteq s\).

\emph{Part 2.} Write \(s_{\max}(v)=V\setminus(\{v\}\cup\operatorname{desc}_G(v))\). It is an ideal: if \(x\in s_{\max}(v)\) and \(u\prec_G x\), then \(u\neq v\) (else \(v\prec_G x\) gives \(x\in\operatorname{desc}_G(v)\), contradicting \(x\in s_{\max}(v)\)); similarly \(u\notin\operatorname{desc}_G(v)\) (else \(v\prec_G u\prec_G x\) gives \(x\in\operatorname{desc}_G(v)\)); so \(u\in s_{\max}(v)\). Since \(v\notin s_{\max}(v)\) and \(\operatorname{anc}_G(v)\subseteq s_{\max}(v)\) (no ancestor of \(v\) is a descendant of \(v\) in a DAG), \(v\in\mathcal A(s_{\max}(v))\). It is reachable: \(s_0\subseteq s_{\max}(v)\) because every \(u\in s_0\) satisfies \(u\neq v\) (since \(v\notin s_0\)) and \(u\notin\operatorname{desc}_G(v)\) (since \(u\in s_0\) and \(v\prec_G u\) would force \(v\in s_0\) by the ideal property, contradicting \(v\notin s_0\)). Finally, for any \(s\in F_v\), if \(u\in s\) then \(u\neq v\) and \(u\notin\operatorname{desc}_G(v)\) (else \(v\prec_G u\) and \(u\in s\) would force \(v\in s\) by the ideal property, contradicting \(v\in\mathcal A(s)\)); hence \(s\subseteq s_{\max}(v)\).
\end{proof}

\begin{corollary}[Analytic \(\Delta_v\) under monotonicity]
\label{cor:supp-analytic-delta}
Under a monotone oracle, for every \(v\in R\),
\[
\Delta_v
=
T_v\!\left[
\frac{1}{p(v,\,s_0\cup\operatorname{anc}_G(v))}
-
\frac{1}{p(v,\,V\setminus(\{v\}\cup\operatorname{desc}_G(v)))}
\right],
\]
and the certified bound \(\sum_{v\in R}\Delta_v\) is computable in \(2|R|\) oracle calls with no lattice enumeration.
\end{corollary}

\begin{proof}
By Lemma~\ref{lem:supp-extremal-states}, \(p(v,\cdot)\) achieves its minimum on \(F_v\) at \(s_{\min}(v)\) and its maximum at \(s_{\max}(v)\) (monotonicity). Since \(c(v,s)=T_v/p(v,s)\) is decreasing in \(p\), the maximum cost is at \(s_{\min}(v)\) and the minimum at \(s_{\max}(v)\); their difference is \(\Delta_v\).
\end{proof}

This result also yields the tightest separable heuristic certificate without enumeration: \(\bar p_v=p(v,s_{\max}(v))=p(v,V\setminus(\{v\}\cup\operatorname{desc}_G(v)))\) is the exact maximum of \(p(v,\cdot)\) on \(F_v\), and is at least as tight as the bound \(p(v,V\setminus\{v\})\) used in the main-text experiments (strictly tighter when \(v\) has descendants).

\section{Proof of Theorem 2: Residual-Acyclic Optimality}
\label{sec:supp-thm2}

Throughout this and the following sections we work on the relabeled fixed domain of Section~\ref{supp:collapse-setup}, writing \(R=V\setminus s_0\) for the remaining concepts and \(\Sigma_G(s_0)\) for the prerequisite-valid complete orderings of \(R\).

\subsection{Explicit Pairwise-Transfer Oracle and Edge Classification}

Theorem~2 of the main paper concerns the explicit nonnegative pairwise-transfer oracle: for fixed weights \(w_{uv}\ge0\) (\(u\neq v\), \(w_{vv}=0\)) and nondecreasing functions \(\varphi_v:[0,\infty)\to(0,1]\),
\[
p(v,s)=\varphi_v\!\Bigl(\sum_{u\in s}w_{uv}\Bigr),
\qquad
c(v,s)=\frac{T_v}{p(v,s)}.
\]
For \(v\in R\), the transfer already supplied by the initial state,
\(b_v^0=\sum_{u\in s_0}w_{uv}\),
is identical across all sequences in \(\Sigma_G(s_0)\), so only transfer edges internal to \(R\) are order-dependent. Let \(H_R=(R,E_H^R)\) collect those edges: \((u,v)\in E_H^R\) iff \(u,v\in R\) and \(w_{uv}>0\). Each edge of \(E_H^R\) falls into exactly one class relative to the hard order \(\prec_G\):
\emph{hard-aligned} (\(u\prec_G v\); satisfied by every valid sequence),
\emph{hard-contradictory} (\(v\prec_G u\); violated by every valid sequence), and
\emph{cross-branch} (\(u,v\) incomparable). Removing the hard-contradictory edges gives
\[
E_H^+=E_H^R\setminus E_H^{\mathrm{contra}},
\qquad
D_R^+=\bigl(R,\;E_G[R]\cup E_H^+\bigr),
\]
the \emph{residual joint graph}. The removed edges are not discretionary losses: no prerequisite-valid sequence can satisfy them.

\subsection{Concept-Wise Attainable Transfer}

For each \(v\in R\) define the attainable transfer upper bound and the induced cost lower bound
\[
\overline{A}_v
=
b_v^0+\!\!\sum_{\substack{u\in R\\ (u,v)\in E_H^+}}\!\! w_{uv},
\qquad
\overline{p}_v=\varphi_v(\overline{A}_v),
\qquad
\underline{c}_v=\frac{T_v}{\overline{p}_v}.
\]

\begin{lemma}[Individual attainability]
\label{lem:supp-individual-attain}
For every fixed \(v\in R\) there exists a prerequisite-valid sequence under which \(v\) receives transfer exactly \(\overline{A}_v\). The sequence may depend on \(v\).
\end{lemma}

\begin{proof}
Let \(U_v=\{u\in R:(u,v)\in E_H^+\}\). By construction no \(u\in U_v\) satisfies \(v\prec_G u\). Augment \(G_R=G[R]\) with the edges \(u\to v\) for all \(u\in U_v\). A directed cycle through a new edge would require a path from \(v\) back to some \(u\in U_v\) in \(G_R\), i.e.\ \(v\prec_G u\), a contradiction; so the augmented graph is acyclic and admits a topological ordering. In that ordering every \(u\in U_v\) precedes \(v\), so together with the fixed initial transfer, \(v\) receives \(b_v^0+\sum_{u\in U_v}w_{uv}=\overline{A}_v\).
\end{proof}

\begin{lemma}[Concept-wise cost lower bound]
\label{lem:supp-cw-lower}
For every \(\sigma\in\Sigma_G(s_0)\) and every \(v\in R\),
\(c(v,s_\sigma(v))\ge\underline{c}_v\),
where \(s_\sigma(v)\) is the state immediately before \(v\) is attempted under \(\sigma\). Consequently \(C(\sigma)\ge C_{\mathrm{cw}}:=\sum_{v\in R}\underline{c}_v\).
\end{lemma}

\begin{proof}
A source \(u\in R\) can precede \(v\) only if \(v\not\prec_G u\), so no hard-contradictory source contributes before \(v\). The transfer received by \(v\) is therefore at most \(\overline{A}_v\). Since \(\varphi_v\) is nondecreasing, \(p(v,s_\sigma(v))\le\varphi_v(\overline{A}_v)=\overline{p}_v\); taking reciprocals and multiplying by \(T_v>0\) gives the per-concept bound, and summing over \(R\) gives the sequence-level bound.
\end{proof}

\subsection{Main Theorem}

\begin{theorem}[Residual-acyclic optimality; Theorem~2 of the main paper]
\label{thm:supp-residual-acyclic}
If \(D_R^+\) is acyclic, then every topological ordering \(\sigma\) of \(D_R^+\) is prerequisite-valid, satisfies every edge of \(E_H^+\), gives every \(v\in R\) transfer exactly \(\overline{A}_v\), and attains \(C(\sigma)=C_{\mathrm{cw}}\). Every topological ordering of \(D_R^+\) is therefore globally optimal, and
\[
C^*(s_0)=\sum_{v\in R}\frac{T_v}{\varphi_v(\overline{A}_v)}.
\]
\end{theorem}

\begin{proof}
Since \(E_G[R]\subseteq E(D_R^+)\), any topological ordering \(\sigma\) of \(D_R^+\) respects all remaining hard prerequisites, and all prerequisites in \(s_0\) are already satisfied; hence \(\sigma\in\Sigma_G(s_0)\). Fix \(v\in R\). Every edge \((u,v)\in E_H^+\) lies in \(D_R^+\), so \(u\) precedes \(v\) under \(\sigma\); together with the sources in \(s_0\), the transfer received by \(v\) is exactly \(\overline{A}_v\), giving \(c(v,s_\sigma(v))=\underline{c}_v\). Summing over \(R\) yields \(C(\sigma)=C_{\mathrm{cw}}\), which by Lemma~\ref{lem:supp-cw-lower} no valid sequence can undercut. The argument applies to every topological ordering of \(D_R^+\).
\end{proof}

The simultaneous attainability of all \(\overline{A}_v\) is a \emph{conclusion} of residual acyclicity, not an assumption of the lower-bound construction. Two scope remarks. First, the theorem is a sufficient optimality certificate, not a characterization: optimal sequences violating some residual soft edge can exist when \(\varphi_v\) is flat over the relevant range or transfer sources are redundant. Second, residual cyclicity is not a hardness certificate for the particular instance; it means only that not all residual preferences can be satisfied simultaneously. Worst-case hardness comes from the construction of the next section.

\section{Proof of Theorem 3: NP-Hardness}
\label{sec:supp-thm3}

\subsection{Source Problem}

A \emph{tournament} is an oriented complete graph \(H=(V,F)\). For a permutation \(\sigma\) of \(V\), an arc \((u,v)\in F\) is \emph{backward} if \(v\) precedes \(u\); \(\operatorname{back}(\sigma)\) counts backward arcs, and \(\operatorname{fwd}(\sigma)=\binom{n}{2}-\operatorname{back}(\sigma)\). Minimum Feedback Arc Set in Tournaments (FAST) asks, given \((H,k)\), whether some permutation has \(\operatorname{back}(\sigma)\le k\); it is NP-hard \citep{alon2006ranking,charbit2007minimum} and, since a permutation is a polynomial-time certificate, NP-complete in decision form.

\subsection{The Restricted Instance Family}

Restricted Ariadne Sequencing (decision form) has input: concepts \(V\) (\(|V|=n\)), a prerequisite DAG \(G\), rational costs \(T_v>0\), a binary matrix \(B=(b_{uv})_{u\neq v}\), and a rational threshold \(K\); the oracle is fixed by the reciprocal-affine formula
\[
p(v,s)=\Bigl(2-\tfrac{1}{n}\sum_{u\in s}b_{uv}\Bigr)^{-1},
\qquad
c(v,s)=2-\tfrac{1}{n}\sum_{u\in s}b_{uv},
\]
and the question is whether some \(\sigma\in\Sigma_G(\varnothing)\) has \(C(\sigma)\le K\). The input is a rational string of length polynomial in \(n\), so NP membership and the reduction below are unambiguous; any general sequencing problem whose input language contains this family inherits its hardness.

\subsection{Construction}

Given a FAST instance \((H=(V,F),k)\): take one concept per vertex; the empty prerequisite DAG \(G=(V,\varnothing)\); unit costs \(T_v=1\); transfer indicators \(b_{uv}=1\) iff \((u,v)\in F\); and threshold
\[
K=2n-\frac{1}{n}\Bigl[\binom{n}{2}-k\Bigr].
\]
The construction runs in \(O(n^2)\) time, and all probabilities evaluate to \(p(v,s)=n/(2n-j)\) with \(j\in\{0,\dots,n-1\}\), so numerators and denominators are bounded by \(2n\).

\begin{lemma}[Instance validity]
\label{lem:supp-hardness-valid}
For every state \(s\) and \(v\notin s\): (i) \(\tfrac12\le p(v,s)\le\tfrac{n}{n+1}<1\); (ii) if \(s\subseteq s'\) and \(v\notin s'\) then \(p(v,s')\ge p(v,s)\).
\end{lemma}

\begin{proof}
(i) Since \(0\le\sum_{u\in s}b_{uv}\le n-1\), the denominator \(2-\frac1n\sum_u b_{uv}\) lies in \([\tfrac{n+1}{n},2]\); take reciprocals. (ii) \(b_{uv}\ge0\) makes the inner sum monotone in \(s\), so the denominator weakly decreases and \(p\) weakly increases.
\end{proof}

Thus hardness will hold with monotone oracles, nonnegative uniform binary transfer, and success probabilities bounded below by the size-independent constant \(1/2\) while remaining strictly below one.

\begin{lemma}[Cost--forward-arc correspondence]
\label{lem:supp-cost-fwd}
For every permutation \(\sigma=(v_1,\dots,v_n)\),
\(C(\sigma)=2n-\frac{1}{n}\operatorname{fwd}(\sigma)\).
Consequently \(C(\sigma)\le C(\sigma')\iff\operatorname{fwd}(\sigma)\ge\operatorname{fwd}(\sigma')\).
\end{lemma}

\begin{proof}
With \(s_{i-1}=\{v_1,\dots,v_{i-1}\}\),
\[
C(\sigma)
=\sum_{i=1}^{n}\Bigl(2-\tfrac1n\!\!\sum_{u\in s_{i-1}}\!\!b_{uv_i}\Bigr)
=2n-\tfrac1n\sum_{i=1}^{n}\sum_{u\in s_{i-1}}b_{uv_i}.
\]
The double sum equals \(\operatorname{fwd}(\sigma)\) by a bijection with forward arcs. If \((u,v)\in F\) is forward, then at the unique index \(i\) with \(v_i=v\) we have \(u\in s_{i-1}\) and \(b_{uv}=1\), contributing exactly one. If \((u,v)\notin F\) then \(b_{uv}=0\); if \((u,v)\in F\) is backward then \(u\notin s_{i-1}\) at that index. So each forward arc is counted once and nothing else is counted.
\end{proof}

\begin{theorem}[NP-completeness; Theorem~3 of the main paper]
\label{thm:supp-hardness}
Restricted Ariadne Sequencing is NP-complete, and its optimization version is NP-hard, even with (1) no prerequisite edges, (2) unit costs, (3) nonnegative transfer, (4) binary uniform transfer magnitude \(1/n\), (5) a monotone oracle, and (6) all success probabilities in \([\tfrac12,1)\).
\end{theorem}

\begin{proof}
\emph{Membership.} A permutation certificate is verified in \(O(n^2)\) rational arithmetic with \(O(\log n)\)-bit entries: check the permutation, compute each \(c(v_i,s_{i-1})=2-j_i/n\), sum, and compare with \(K\).

\emph{Hardness.} With the construction above and Lemma~\ref{lem:supp-cost-fwd},
\[
C(\sigma)\le K
\iff
\operatorname{fwd}(\sigma)\ge\binom{n}{2}-k
\iff
\operatorname{back}(\sigma)\le k,
\]
so the FAST instance is a yes-instance iff the constructed instance is. Restrictions (1)--(6) hold by construction and Lemma~\ref{lem:supp-hardness-valid}.
\end{proof}

Each restriction rules out an artificial source of hardness: prerequisite feasibility (empty DAG), heterogeneous costs (unit \(T_v\)), interference (nonnegative transfer), numerical encoding tricks (binary uniform magnitude), and extreme probabilities (constant lower bound \(1/2\)). Hardness arises from state-conditioned cross-branch transfer alone. The construction sits at maximal poset width \(w=n\) (every subset is an ideal), which is exactly the regime excluded by the width bound of the next section.

\section{Proof of Proposition 2: Width-Parameterized Dynamic Programming}
\label{sec:supp-prop2}

Let \(P=(V,\preceq_G)\) be the prerequisite poset restricted to the \(m\) remaining concepts, and let \(w\) be its width (the size of a largest antichain).

\begin{proposition}[Width-XP; Proposition~2 of the main paper]
\label{prop:supp-width-xp}
The collapsed sequencing problem is solvable exactly by reverse-topological dynamic programming in
\(O\bigl(m\,(m/w+1)^{w}\bigr)\) time and \(O\bigl((m/w+1)^{w}\bigr)\) states. For fixed \(w\) this is polynomial in \(m\); the problem is in XP parameterized by \(w\).
\end{proposition}

\begin{proof}
By Dilworth's theorem \citep{dilworth1950decomposition}, \(P\) partitions into \(w\) chains \(C_1,\dots,C_w\) of lengths \(n_1,\dots,n_w\) with \(\sum_i n_i=m\). Every order ideal intersects each chain in a prefix, so the map from ideals to progress vectors \((k_1,\dots,k_w)\), \(0\le k_i\le n_i\), is injective (not every vector need be an ideal, because cross-chain precedences may exclude combinations). Hence
\[
|\mathcal I|\;\le\;\prod_{i=1}^{w}(n_i+1)\;\le\;\Bigl(\frac{m}{w}+1\Bigr)^{w},
\]
the second inequality by AM--GM. The reverse-cardinality Bellman recursion
\(J(s)=\min_{v\in\mathcal A(s)}[\,c(v,s)+J(s\cup\{v\})\,]\)
evaluates each reachable ideal once, scanning at most \(m\) candidate actions per state, giving the stated time bound; storage is one value per ideal.
\end{proof}

Two boundary remarks. The bound is XP, not FPT: the exponent depends on the parameter, and no algorithm of the form \(f(w)\,m^{O(1)}\) is claimed. The special cases anchor the range: at \(w=1\) the poset is a chain with \(m+1\) ideals and a unique valid sequence, so optimization is trivial; at \(w=m\) (empty DAG) every subset is an ideal, \(|\mathcal I|=2^m\), and this is precisely the regime of the hardness construction in Section~\ref{sec:supp-thm3}.

\section{Proof of Theorem 4: Fixed-Parameter Tractability for Uniform Affine Transfer}
\label{sec:supp-thm4}

We retain the residual graph \(D_R^+=(R,E_G[R]\cup E_H^+)\) of Section~\ref{sec:supp-thm2}. For \(\sigma\in\Sigma_G(s_0)\), let \(B(\sigma)\subseteq E_H^+\) denote the residual soft edges violated by \(\sigma\) (source placed after target).

\subsection{Distance to Residual Acyclicity}

Define the residual feedback-arc distance
\[
k
=
\min\bigl\{|F|:F\subseteq E_H^+,\;(R,\,E_G[R]\cup(E_H^+\setminus F))\text{ acyclic}\bigr\}.
\]

\begin{lemma}[Ordering/deletion equivalence]
\label{lem:supp-ord-del}
\(k=\min_{\sigma\in\Sigma_G(s_0)}|B(\sigma)|\).
\end{lemma}

\begin{proof}
(\(\le\)) For any \(\sigma\in\Sigma_G(s_0)\), deleting \(F=B(\sigma)\) leaves only edges pointing forward under \(\sigma\), hence an acyclic graph; so \(k\le|B(\sigma)|\), and minimizing over \(\sigma\) gives one direction. (\(\ge\)) For any deletion set \(F\) restoring acyclicity, a topological ordering \(\sigma_F\) of the remaining graph is prerequisite-valid and violates only edges in \(F\), so \(\min_\sigma|B(\sigma)|\le|B(\sigma_F)|\le|F|\); minimize over \(F\).
\end{proof}

\begin{lemma}[FPT of residual frustration]
\label{lem:supp-fpt-frustration}
Deciding whether \(k\le k_0\), and computing a prerequisite-valid ordering with \(|B(\sigma)|=k\), is fixed-parameter tractable in \(k_0\).
\end{lemma}

\begin{proof}
The problem is a Directed Subset Feedback Arc Set instance in which only arcs of \(E_H^+\) are deletable while arcs of \(E_G[R]\) must be kept. In the point-algebra formulation, each residual soft edge \(u\to v\) is a breakable strict constraint \(u<v\) and each hard edge an unbreakable constraint \(u\le v\); a feasible assignment with at most \(k_0\) broken strict constraints corresponds exactly (via Lemma~\ref{lem:supp-ord-del}) to a valid ordering with at most \(k_0\) violations. \(\operatorname{MinCSP}(<,\le)\), equivalently Directed Subset Feedback Arc Set, is FPT in the number of violated constraints \citep{osipov2024point}. An optimal ordering is recovered by standard self-reduction (polynomially many decision calls) followed by topological sorting of the retained relation.
\end{proof}

\subsection{The Uniform Affine Identity}

Assume uniform affine collapsed costs: after absorbing fixed initial-transfer and unavoidable contradictory-transfer contributions into constants \(a_v\),
\[
c(v,s)=a_v-\lambda\!\!\sum_{\substack{u\in s\cap R\\ (u,v)\in E_H^+}}\!\!1,
\qquad \lambda>0,\quad c(v,s)>0\text{ on all feasible states.}
\]
The hardness family of Section~\ref{sec:supp-thm3} is the special case \(\lambda=1/n\).

\begin{lemma}[Cost/frustration identity]
\label{lem:supp-affine-identity}
For every \(\sigma\in\Sigma_G(s_0)\),
\(C(\sigma)=C_0+\lambda\,|B(\sigma)|\)
with the sequence-independent constant \(C_0=\sum_{v\in R}a_v-\lambda|E_H^+|\).
\end{lemma}

\begin{proof}
Let \(N^-_\sigma(v)\) be the satisfied residual in-edges of \(v\) under \(\sigma\). Then \(C(\sigma)=\sum_{v}(a_v-\lambda|N^-_\sigma(v)|)\), and since each satisfied edge is counted exactly once, \(\sum_v|N^-_\sigma(v)|=|E_H^+|-|B(\sigma)|\). Substitute.
\end{proof}

\begin{theorem}[Affine FPT; Theorem~4 of the main paper]
\label{thm:supp-affine-fpt}
Under uniform affine collapsed costs, optimal sequencing is fixed-parameter tractable in the residual feedback-arc distance \(k\): an optimal prerequisite-valid sequence is computable in \(f(k)\cdot(|R|+|E_G|+|E_H|)^{O(1)}\) time, where the polynomial is in the encoding of the residual graph, and \(C^*=C_0+\lambda k\).
\end{theorem}

\begin{proof}
By Lemma~\ref{lem:supp-affine-identity}, cost minimization coincides with residual-frustration minimization; by Lemma~\ref{lem:supp-fpt-frustration}, a minimum-frustration ordering is computable in FPT time; Lemma~\ref{lem:supp-ord-del} identifies the optimum value of the frustration with \(k\).
\end{proof}

At \(k=0\) the result degenerates to the linear-time topological solution of Theorem~\ref{thm:supp-residual-acyclic}, while unbounded \(k\) admits the NP-hard family of Section~\ref{sec:supp-thm3}; small distance to acyclicity thus interpolates exactly between the two anchors. The FPT conclusion is claimed \emph{only} for the uniform affine subclass (and, with the optimum integer violation cost as the parameter, its positive-integer-weighted extension): for a general nonlinear monotone pairwise oracle, minimizing frustration need not minimize cost, and no analogous parameterized algorithm is asserted.

\section{Proof of Proposition 3: Consistency of the Sum Heuristic}
\label{sec:supp-prop3}

Recall from the main paper that \(F_v\) denotes the reachable states in which \(v\) is feasible, and that \(\bar p_v\) is any certified upper bound with \(p(v,s)\le\bar p_v\) for all \(s\in F_v\). The sum heuristic is
\[
h_{\mathrm{sum}}(s)=\sum_{u\notin s}\frac{T_u}{\bar p_u},
\qquad
h_{\mathrm{sum}}(s^*)=0.
\]

\begin{proposition}[Consistency; Proposition~3 of the main paper]
\label{prop:supp-hsum}
\(h_{\mathrm{sum}}\) is consistent on the collapsed lattice:
\(h_{\mathrm{sum}}(s)\le c(v,s)+h_{\mathrm{sum}}(s\cup\{v\})\)
for every reachable \(s\) and \(v\in\mathcal A(s)\). Consequently it is admissible, and graph-search A* with \(h_{\mathrm{sum}}\) never reopens an expanded state.
\end{proposition}

\begin{proof}
Crossing the edge \(s\to s\cup\{v\}\) removes exactly the term of \(v\) from the sum:
\[
h_{\mathrm{sum}}(s)-h_{\mathrm{sum}}(s\cup\{v\})=\frac{T_v}{\bar p_v}
\le
\frac{T_v}{p(v,s)}=c(v,s),
\]
where the inequality uses \(p(v,s)\le\bar p_v\) (valid since \(s\in F_v\)) and \(T_v>0\). Rearranging gives the consistency inequality. Admissibility follows by summing the inequality along any path to the goal and using \(h_{\mathrm{sum}}(s^*)=0\). With a consistent heuristic, \(f=g+h_{\mathrm{sum}}\) is nondecreasing along every path, so when a state is removed from the priority queue its \(g\)-value is optimal and it is never reopened; the first extraction of \(s^*\) returns the optimal cost.
\end{proof}

The proof uses only that each collapsed edge adds exactly one concept, that the heuristic is separable over remaining concepts, and that \(\bar p_v\) is a valid upper bound on \(F_v\); monotonicity of the oracle is not required. Valid bound choices trade tightness against precomputation: the trivial \(\bar p_v=1\) (yielding \(h_{\mathrm{sum}}(s)=\sum_{u\notin s}T_u\)); under a monotone oracle, the analytic bound \(\bar p_v=p(v,V\setminus\{v\})\); and the tightest separable choice \(\bar p_v=p(v,V\setminus(\{v\}\cup\operatorname{desc}_G(v)))\), which under monotonicity is also analytic (Corollary~\ref{cor:supp-analytic-delta}) and strictly tighter than \(p(v,V\setminus\{v\})\) when \(v\) has descendants. For non-monotone oracles, the tightest separable bound may require enumeration. All choices preserve consistency; the experiments in the main paper use the intermediate monotone bound \(p(v,V\setminus\{v\})\). The result can fail outside this structure, for example if one action masters several concepts, mastery can be lost, or edge costs contain cross-concept interaction terms not represented in the separable form.

\section{Dataset, Canonical Preprocessing, and Student Split}
\label{sec:supp-data-preprocessing}

\subsection{Raw Export and Interaction Cleaning}

The empirical study used a frozen, anonymized PrairieLearn submission export from an introductory computer science course. The raw export contained 83,739 submission records from 296 anonymized students across one ten-week instructional quarter, covering 144 distinct question identifiers. Its relevant fields included the anonymized student identifier, question identifier, timestamp, correctness indicator, score, and attempt number.

The cleaning pipeline removed 12,389 incomplete records for which both the correctness indicator and score were missing. It then recomputed attempt numbers within each student--question history, checked that timestamps and outcome fields were valid, and deduplicated records by student, question, and timestamp. No duplicate was present in the frozen input, so this stage produced 71,350 cleaned interactions from the same 296 students and 144 questions.

Questions were then joined to the frozen question--concept mapping. The mapping associated 141 question identifiers with 34 observed concept nodes in the 61-node prerequisite DAG and contained neither duplicate question identifiers nor out-of-graph concept nodes. The join excluded 457 interactions from three question identifiers absent from the final mapping. The canonical mapped dataset therefore contained 70,893 interactions from 294 students. These counts distinguish raw submission cleaning from concept mapping: the former produced the 71,350-row interaction file, whereas the latter produced the 70,893 interactions used by the knowledge-tracing and sequencing experiments.

\subsection{Canonical Concept Sessions}

Mapped interactions were stably ordered by student identifier, timestamp, and original source order. For each student, every maximal contiguous block of interactions mapped to the same concept was aggregated into one canonical concept session. A change of student or concept opened a new session; elapsed time alone did not split a block. This run-length aggregation produced 31,559 sessions.

For a session \(j\) containing \(n_j\) binary interaction outcomes \(y_{j1},\ldots,y_{jn_j}\), the session score was
\[
    \operatorname{score}_j
    =
    \frac{1}{n_j}
    \sum_{\ell=1}^{n_j} y_{j\ell}.
\]
The binary outcome supplied to BKT and DKT was
\[
    \operatorname{correct}_j
    =
    \mathbf{1}\!\left\{\operatorname{score}_j \geq 0.8\right\}.
\]
Thus, the knowledge-tracing models received one binary observation per canonical session rather than a fractional score or one observation per raw submission. Likewise, the training-split \texttt{attempt\_count} exposed in the LLM Full condition counted canonical sessions.

\subsection{Prefix Mastery and Prerequisite Closure}

Mastery states were constructed causally from the sessions preceding each evaluation prefix. For every student--concept pair, the implementation maintained an independent success streak. A session with score at least \(0.8\) incremented the streak, while a lower-scoring session reset it. A concept became mastered after three consecutive successful sessions and remained mastered thereafter.

The observed mastery set was then projected onto the order-ideal state space of the prerequisite DAG. If a mastered descendant required an ancestor that had training observations but lacked mastery evidence in the prefix, the descendant was removed rather than silently completing that observed ancestor. Automatic prerequisite completion was restricted to ancestors with zero observations in the training split. The resulting state was finally checked for prerequisite closure. This rule preserves the planner's state invariant without treating a missing mastery event on an empirically observed prerequisite as a success.

\subsection{Student-Level Split and Leakage Controls}

Splitting was performed at the student level, so all sessions from a student belonged to exactly one partition. The 294 student identifiers were first ordered lexicographically and then permuted with \texttt{numpy.random.default\_rng(42)}. With \(\lfloor 0.1N\rfloor=29\), the first 29 permuted students formed the test set, the next 29 formed the validation set, and the remaining 236 formed the training set. Table~\ref{tab:supp-preprocessing-counts} reports the frozen counts.

\begin{table}[t]
    \centering
    \small
    \caption{Frozen preprocessing and student-split counts. Interaction and session counts refer to distinct stages and should not be conflated.}
    \label{tab:supp-preprocessing-counts}
    \begin{tabular}{@{}l r r@{}}
        \toprule
        Stage or split & Students & Records or sessions \\
        \midrule
        Raw PrairieLearn export & 296 & 83,739 interactions \\
        Cleaned interactions & 296 & 71,350 interactions \\
        Concept-mapped interactions & 294 & 70,893 interactions \\
        Canonical dataset & 294 & 31,559 sessions \\
        \midrule
        Training split & 236 & 25,089 sessions \\
        Validation split & 29 & 3,103 sessions \\
        Test split & 29 & 3,367 sessions \\
        \bottomrule
    \end{tabular}
\end{table}

All parameter estimation for BKT used only the training split. DKT gradient updates likewise used only training students; validation students were used for early stopping and checkpoint selection, and test students were not used for either fitting or model selection. The FrequencyOracle was estimated from canonical training sessions and evaluated on validation data only for reporting. The aggregate \texttt{attempt\_count} and \texttt{success\_rate} supplied to the LLM Full conditions were also computed exclusively from canonical training sessions. No raw student history, validation statistic, or test statistic was exposed to an LLM condition.

\section{Shared Objective and Public Evaluator}
\label{sec:supp-shared-evaluator}

All sequencing conditions were compared under the same risk-neutral expected-cost objective. Each attempt on a concept \(v\) incurred the uniform nominal cost
\begin{equation}
    T_v = 60.0,
    \qquad \text{for every concept } v.
\end{equation}
The value of \(60.0\) seconds provides a common scale for reporting costs; it is not an empirical estimate of the time required to complete a concept. In particular, we did not assign costs using concept-specific mean or median completion times. Because the same cost is attached to every attempt, minimizing expected total cost is equivalent to minimizing the expected number of attempts, up to the constant factor of \(60.0\).

We write \(C_F\) for the quantity denoted \(C_{\mathrm{eval}}\) in the main text, and \(C_F^*\) for \(J^*_{\mathrm{eval}}\), to emphasize dependence on the frozen oracle \(F\).

For a fixed target closure, let \(\sigma=(v_1,\ldots,v_m)\in\Sigma_G(s_0)\) denote a prerequisite-valid complete ordering. For \(i=1,\ldots,m\), let
\begin{equation}
    s_i = s_0 \cup \{v_1,\ldots,v_i\}
\end{equation}
be the mastery state induced by its first \(i\) concepts; the initial state was \(s_0=\varnothing\) in all reported experiments. At step \(i\), the public evaluator queries the frozen success probability \(p_F(v_i,s_{i-1})\) supplied by the \texttt{FrozenMonotonicOracle}. Under the retry-until-success dynamics of the instructional model, an attempt that fails leaves the mastery state unchanged. The expected number of attempts required to master \(v_i\) is therefore \(1/p_F(v_i,s_{i-1})\), and the corresponding collapsed expected cost is \(T_{v_i}/p_F(v_i,s_{i-1})\). The evaluation cost of the complete ordering is consequently
\begin{equation}
    C_F(\sigma)
    =
    \sum_{i=1}^{m}
    \frac{T_{v_i}}{p_F(v_i,s_{i-1})}.
    \label{eq:supp-frozen-evaluation-cost}
\end{equation}

The public evaluator was executed deterministically on CPU with a fixed oracle checkpoint, prerequisite graph, preprocessing artifacts, and cost configuration. These quantities were not updated between conditions or repetitions. For each target closure, exact dynamic programming used the same oracle and the same state-transition costs as in Eq.~\eqref{eq:supp-frozen-evaluation-cost} to obtain the evaluator-defined optimum
\begin{equation}
    C_F^*
    =
    \min_{\sigma\in\Sigma_G(s_0)} C_F(\sigma).
\end{equation}

For every sequence \(\sigma\in\Sigma_G(s_0)\), normalized regret was computed as
\begin{equation}
    R_F(\sigma)
    =
    \frac{C_F(\sigma)-C_F^*}{C_F^*}.
    \label{eq:supp-normalized-regret}
\end{equation}
Differences satisfying
\(
|C_F(\sigma)-C_F^*|
\leq 10^{-12}\max\{1,C_F^*\}
\)
were treated as numerical ties and recorded as zero regret (the unit guard handles the degenerate case \(C_F^*=0\)).

Every submitted sequence was rescored by this evaluator, irrespective of how it had been produced. The internal beliefs used by the Frequency, BKT-derived, DKT-derived, and LLM conditions affected sequence selection but did not alter the evaluation probabilities or costs. This shared scoring rule makes results comparable across conditions, while also delimiting their interpretation. The reported regret measures deviation from the optimum defined by a fixed simulated learner model. It does not, by itself, estimate a causal improvement in learning time or learning outcomes for students deployed under the corresponding sequence.

\section{Conditions and Run Matrix}
\label{sec:supp-conditions}

The formal comparison comprised fourteen sequencing conditions evaluated on the same ten prespecified target closures. All conditions used the common empty initial mastery state and the prerequisite graph fixed by the experiment manifest. A condition was responsible only for producing an ordering; every resulting ordering was subsequently validated and scored under the shared protocol in Section~\ref{sec:supp-shared-evaluator}. This separation allowed methods with different internal learner models or information sources to be compared without changing the evaluation objective.

The planner-based conditions crossed four probability models with either a local Greedy rule or Exact optimization. Ariadne used the \texttt{FrozenMonotonicOracle}, which was also the public evaluator. Frequency used a state-independent population success prior. The BKT-set and DKT-set conditions used deterministic Set Oracle surrogates distilled from their respective history-dependent knowledge-tracing teachers. Greedy selected from the current prerequisite-valid frontier using the corresponding oracle's immediate collapsed cost, whereas Exact optimized the same oracle's multistep planning objective.

Two additional conditions supplied structural reference points without evaluator-aware optimization. Random Frontier sampled uniformly from the concepts available on the current prerequisite frontier at each step. It therefore did not sample uniformly from the set of complete topological orderings. Linear Syllabus restricted the official course order to the nodes in the relevant target closure.

The four LLM conditions crossed two models with Zero and Full information settings. Zero received anonymized concept identifiers together with the prerequisite edges and hence had access to curriculum structure but no student data. Full received the identical graph input plus each concept's training-split \texttt{attempt\_count} and binary \texttt{success\_rate}. It did not receive raw student histories, validation or test statistics, empirical completion times, concept-specific durations, evaluator queries, or previously generated planner sequences. Each LLM run was a fresh, stateless, single-turn request, and an invalid response was neither repaired nor regenerated.

Table~\ref{tab:supp-run-matrix} gives the complete planned run matrix. Each deterministic planner or syllabus condition produced one sequence for each target. Random Frontier produced 100 independent samples per target, while each LLM condition used 20 independent logical repetitions per target. The resulting design contained 1,890 planned runs in total.

\begin{table*}[t]
    \centering
    \small
    \setlength{\tabcolsep}{5pt}
    \caption{Formal sequencing conditions and planned run counts across the ten target closures. Exact rows were computed by the LAO*-derived solver; graph-search A* returns identical values (Section~5 of the main paper).}
    \label{tab:supp-run-matrix}
    \begin{tabular}{@{}lll r@{}}
        \toprule
        Family & Condition & Planner or information used & Planned runs \\
        \midrule
        Ariadne & Ariadne Greedy & \texttt{FrozenMonotonicOracle} with Greedy & 10 \\
        Ariadne & Ariadne Exact & \texttt{FrozenMonotonicOracle} with Exact & 10 \\
        Frequency & Frequency Greedy & State-independent population prior with Greedy & 10 \\
        Frequency & Frequency Exact & State-independent population prior with Exact & 10 \\
        BKT-set & BKT-set Greedy & BKT-derived Set Oracle with Greedy & 10 \\
        BKT-set & BKT-set Exact & BKT-derived Set Oracle with Exact & 10 \\
        DKT-set & DKT-set Greedy & DKT-derived Set Oracle with Greedy & 10 \\
        DKT-set & DKT-set Exact & DKT-derived Set Oracle with Exact & 10 \\
        Structural baseline & Random Frontier & Uniform sampling from the current valid frontier & 1,000 \\
        Structural baseline & Linear Syllabus & Official syllabus order restricted to the closure & 10 \\
        LLM & GPT-5.6 SOL Zero & Anonymized concepts and prerequisite edges & 200 \\
        LLM & GPT-5.6 SOL Full & Graph plus attempt count and success rate & 200 \\
        LLM & DeepSeek V4 Pro Zero (preview) & Anonymized concepts and prerequisite edges & 200 \\
        LLM & DeepSeek V4 Pro Full (preview) & Graph plus attempt count and success rate & 200 \\
        \midrule
        \multicolumn{3}{r}{Total planned runs} & 1,890 \\
        \bottomrule
    \end{tabular}
\end{table*}

Before scoring, each output was required to cover the target closure exactly once, contain no duplicate or out-of-closure concepts, respect every prerequisite edge, and place the target concept last. Runs that failed these structural checks retained their planned-run identity but had no evaluation cost or regret. Consequently, the planned counts in Table~\ref{tab:supp-run-matrix} describe the experimental design rather than the number of observations entering the conditional cost summaries.

\FloatBarrier

\section{Oracle Validation}
\label{sec:supp-oracle-validation}

We evaluate the frozen Oracle artifacts on the validation split. Regression errors use all 3,103 continuous session scores. AUC and threshold-\(0.5\) accuracy use only the 2,170 sessions whose aggregated score is exactly zero or one; mixing these denominators would make the metrics appear more directly comparable than they are.

\begin{table*}[t]
    \centering
    \small
    \caption{Held-out Oracle metrics. \(N_{\mathrm{reg}}\) is the continuous-label sample count and \(N_{\mathrm{bin}}\) is the binary-label sample count used for AUC and accuracy.}
    \label{tab:supp-oracle-validation}
    \begin{tabular}{@{}lrrrrrrr@{}}
        \toprule
        Oracle/input mode & \(N_{\mathrm{reg}}\) & MSE & RMSE & MAE & \(N_{\mathrm{bin}}\) & AUC & Accuracy \\
        \midrule
        Frozen, full feature & 3,103 & 0.1237 & 0.3517 & 0.3033 & 2,170 & 0.7749 & 0.7986 \\
        Frozen, zero-history ablation & 3,103 & 0.2507 & 0.5007 & 0.4365 & 2,170 & 0.6110 & 0.4028 \\
        FrequencyOracle & 3,103 & 0.1375 & 0.3708 & 0.3206 & 2,170 & 0.6625 & 0.7797 \\
        \bottomrule
    \end{tabular}
\end{table*}

The full-feature model receives the fixed prerequisite DAG, target-node index, and two features for every concept: whether the concept has appeared in the student's previous sessions and its latest session score. Its mask is a copy of the appeared-concept indicator. At deployment, the planner adapter instead sets both history features to zero and supplies the prerequisite-closed planner state as the mask. The network also uses learned node embeddings and the fixed graph structure; it receives no student identifier, timestamp, attempt count, or ordered interaction sequence.

The artifact calls the second row \texttt{planning\_mode}, but its validation implementation zeroes the history features while retaining masks constructed from concepts that merely appeared in each observed prefix. Those masks are not the planner's irreversible, prerequisite-closed mastery states: only 29 of 3,103 validation-prefix masks are prerequisite-closed (193 unique masks). Consequently, AUC \(0.6110\) is a reproducible zero-history-feature ablation, not an estimate of predictive performance over the legal planner-state distribution. FrequencyOracle has training global mean success probability \(0.7042\).

\begin{table}[t]
    \centering
    \small
    \caption{Binary-label calibration on the same 2,170 validation sessions. ECE and MCE use ten equal-width probability bins.}
    \label{tab:supp-oracle-calibration}
    \begin{tabular}{@{}lrrr@{}}
        \toprule
        Frozen input mode & Brier score & ECE & MCE \\
        \midrule
        Full feature & 0.1553 & 0.1176 & 0.2204 \\
        Zero-history ablation & 0.3315 & 0.3851 & 0.4872 \\
        \bottomrule
    \end{tabular}
\end{table}

For calibration bin \(b\), ECE is \(\sum_b (n_b/N)\lvert\bar p_b-\bar y_b\rvert\), and MCE is the largest bin gap. These bin-dependent diagnostics show that the full-feature predictor is materially better calibrated than the zero-history ablation; they do not repair the planner-state support mismatch described above.

\begin{table*}[t]
    \centering
    \small
    \caption{Held-out fidelity of the BKT- and DKT-derived Set-Oracle surrogates. Both use the same 29 validation students, 3,103 prefixes, seven prerequisite-closed mastery states, and 189 state--target pairs.}
    \label{tab:supp-set-oracle-validation}
    \begin{tabular}{@{}lrrrrr@{}}
        \toprule
        Surrogate & Grouped MSE & Prefix MSE & Prefix MAE & Pearson & Spearman \\
        \midrule
        BKT-derived & 0.002041 & 0.012701 & 0.064823 & 0.7643 & 0.7229 \\
        DKT-derived & 0.000481 & 0.014172 & 0.088960 & 0.5792 & 0.5697 \\
        \bottomrule
    \end{tabular}
\end{table*}

The BKT-derived surrogate uses concept-specific parameters where training observations exist and training-only pooled-parameter backoff for eight required zero-observation nodes. Grouped MSE measures the grouped surrogate objective and is not numerically equivalent to raw prefix-level loss. The correlation and prefix-error columns must therefore be interpreted together; the smaller DKT grouped MSE alone does not establish greater overall teacher fidelity. These are teacher--surrogate agreement metrics, not predictive errors against observed student outcomes. Moreover, seven observed mastery states delimit the empirical support of both diagnostics, so these results do not establish equal fidelity over every reachable order ideal.

\section{Main Sequencing Results}
\label{sec:supp-main-results}

Results were aggregated in two stages so that target closures, rather than raw repetitions, defined the comparison unit. Repetitions were first averaged within each target, after which the ten target-level means were weighted equally. Evaluation cost and normalized regret were computed only for runs whose output sequence belonged to \(\Sigma_G(s_0)\) (prerequisite-valid, complete, no duplicates). The valid-to-planned counts are reported alongside these conditional performance measures so that sequence quality is not separated from pipeline reliability.

Table~\ref{tab:supp-main-results} reports the complete fourteen-condition comparison under the frozen public evaluator. Mean evaluation costs occupy the narrow interval from \(1789.6680\) to \(1790.3653\) nominal seconds, and the largest mean normalized regret is \(3.385\times10^{-4}\), or approximately \(0.034\%\). The ranking should therefore be read as a comparison within a low-regret regime rather than as evidence of large absolute differences in expected learning cost.

\begin{table*}[t]
    \centering
    \small
    \setlength{\tabcolsep}{8pt}
    \caption{Target-equal sequence quality under the shared \texttt{FrozenMonotonicOracle} evaluator. Cost and regret are conditional on sequences in \(\Sigma_G(s_0)\). Exact rows were computed by the LAO*-derived solver; graph-search A* returns identical values (Section~5 of the main paper).}
    \label{tab:supp-main-results}
    \begin{tabular}{@{}r l r r r@{}}
        \toprule
        Regret rank & Condition & Valid / planned & Mean eval.\ cost & Mean norm.\ regret \\
        \midrule
         1 & Ariadne Exact & 10 / 10 & 1789.6680 & \(0\) \\
         2 & Ariadne Greedy & 10 / 10 & 1789.6694 & \(5.934\times10^{-7}\) \\
         3 & Frequency Greedy & 10 / 10 & 1789.6929 & \(1.047\times10^{-5}\) \\
         4 & DKT-set Greedy & 10 / 10 & 1789.7747 & \(4.935\times10^{-5}\) \\
         5 & GPT-5.6 SOL Full & 200 / 200 & 1789.8193 & \(6.599\times10^{-5}\) \\
         6 & DeepSeek V4 Pro Full (preview) & 197 / 200 & 1789.9291 & \(1.221\times10^{-4}\) \\
         7 & GPT-5.6 SOL Zero & 200 / 200 & 1790.0019 & \(1.619\times10^{-4}\) \\
         8 & DKT-set Exact & 10 / 10 & 1790.0333 & \(1.828\times10^{-4}\) \\
         9 & BKT-set Greedy & 10 / 10 & 1790.0933 & \(1.995\times10^{-4}\) \\
        10 & DeepSeek V4 Pro Zero (preview) & 195 / 200 & 1790.1102 & \(2.089\times10^{-4}\) \\
        11 & BKT-set Exact & 10 / 10 & 1790.1889 & \(2.481\times10^{-4}\) \\
        12 & Random Frontier & 1,000 / 1,000 & 1790.2230 & \(2.633\times10^{-4}\) \\
        13 & Frequency Exact & 10 / 10 & 1790.2100 & \(2.658\times10^{-4}\) \\
        14 & Linear Syllabus & 10 / 10 & 1790.3653 & \(3.385\times10^{-4}\) \\
        \bottomrule
    \end{tabular}
\end{table*}

\FloatBarrier

\subsubsection{Evaluator-matched Ariadne Conditions}

Ariadne Exact attained zero normalized regret because its planning oracle and objective were identical to those used by the public evaluator. Its result is therefore a protocol identity: it realizes the evaluator-defined optimum by construction. The result should not be interpreted as an independently observed improvement in real-student outcomes. Ariadne Greedy was also nearly optimal, with mean normalized regret \(5.934\times10^{-7}\). On these ten closures and under this frozen oracle, local choices were sufficient to reproduce the global optimum on nine targets and differed only slightly on the remaining target. This empirical near-equivalence does not imply that Greedy and Exact are generally interchangeable.

\subsubsection{Planning under Mismatched Internal Oracles}

The public-evaluator ranking of Greedy and Exact reversed for each mismatched planner family. Frequency Greedy and Frequency Exact obtained mean regrets of \(1.047\times10^{-5}\) and \(2.658\times10^{-4}\), respectively. The corresponding comparisons were \(1.995\times10^{-4}\) versus \(2.481\times10^{-4}\) for BKT-set and \(4.935\times10^{-5}\) versus \(1.828\times10^{-4}\) for DKT-set. These values do not show that the Exact solver failed to optimize the objective supplied to it. Rather, the planning objective and the final scoring objective were different.

The internal planner checks make this distinction explicit. Under both the BKT-derived and DKT-derived Set Oracles, the Exact solver strictly improved upon Greedy on eight targets and tied it on two. Its internal cost also matched exact dynamic programming on every target, with zero absolute gap. The largest internal reduction occurred for target 39 and was \(29.056\) nominal seconds under the BKT-derived oracle and \(24.175\) under the DKT-derived oracle. Table~\ref{tab:supp-oracle-mismatch} contrasts this verified internal optimality with the ordering induced by the external evaluator.

\begin{table*}[t]
    \centering
    \small
    \caption{Greedy--Exact comparison under each Set Oracle's internal objective and under the frozen public evaluator.}
    \label{tab:supp-oracle-mismatch}
    \begin{tabular}{@{}p{0.18\textwidth}p{0.44\textwidth}p{0.30\textwidth}@{}}
        \toprule
        Planning oracle & Internal objective & Public-evaluator mean regret \\
        \midrule
        BKT-derived Set Oracle &
        Exact improved 8 of 10 targets and tied 2; its largest cost reduction was \(29.056\) at target 39, and its exact-DP gap was zero for every target. &
        Greedy \(1.995\times10^{-4}\); Exact \(2.481\times10^{-4}\). \\
        DKT-derived Set Oracle &
        Exact improved 8 of 10 targets and tied 2; its largest cost reduction was \(24.175\) at target 39, and its exact-DP gap was zero for every target. &
        Greedy \(4.935\times10^{-5}\); Exact \(1.828\times10^{-4}\). \\
        \bottomrule
    \end{tabular}
\end{table*}

\FloatBarrier

These reversals are consequently evidence of oracle mismatch rather than solver incorrectness. More complete optimization of an approximate internal learner model need not produce a sequence closer to the optimum of a different external learner model; it can instead amplify the consequences of model error. The Frequency oracle is state independent, so its internal objective contains no order-sensitive student-state information. Different deterministic tie resolutions can nevertheless produce different sequences and therefore different costs when those sequences are rescored by the state-dependent public evaluator.

\subsubsection{Structural Reference Conditions}

Random Frontier and Linear Syllabus produced prerequisite-valid sequences without using the public evaluator to optimize their orderings. Random Frontier achieved mean normalized regret \(2.633\times10^{-4}\). Because it sampled uniformly from the currently feasible frontier at each step, its distribution was not uniform over complete topological orderings. Linear Syllabus, obtained by restricting the official instructional order to each target closure, achieved mean regret \(3.385\times10^{-4}\). The two conditions therefore provide distinct reference points: one preserves only local prerequisite validity, while the other preserves a fixed course-level ordering.

All fourteen conditions remained close to the frozen optimum in absolute terms. This common low-regret scale is itself an experimental finding: for the selected ECS32A closures, the evaluator assigns similar total cost to many prerequisite-valid sequences. The table supports comparisons of sequencing behavior under the frozen learner model, but it does not establish corresponding causal gains for students without a real-student intervention.

\section{Mean Evaluation Costs}
\label{sec:supp-mean-costs}

Table~\ref{tab:supp-mean-costs} reports the per-target mean evaluation cost for all fourteen conditions, together with the evaluator-defined optimum $C_F^*$ for each target closure. Costs are in nominal seconds ($T_v=60$). For conditions with multiple repetitions (Random Frontier: 100 seeds; each LLM condition: 20 runs), the entry is the arithmetic mean over valid runs for that target; deterministic conditions have a single observation per target.

All costs fall in a narrow band: across all 140 condition--target cells, the maximum cost is $2500.38$ and the minimum is $536.55$, while the maximum deviation from the corresponding $C_F^*$ is $0.98$ nominal seconds (Linear Syllabus, target~18). The table therefore documents a uniformly low-regret regime rather than meaningful cost separation.

\begin{table*}[t]
\centering
\scriptsize
\setlength{\tabcolsep}{3.5pt}
\caption{Per-target mean evaluation cost (nominal seconds) for all conditions. $C_F^*$ is the evaluator-defined optimum. Exact rows were computed by the LAO*-derived solver; graph-search A* returns identical values (Section~5 of the main paper).}
\label{tab:supp-mean-costs}
\begin{tabular}{@{}l rrrrrrrrrr@{}}
\toprule
Condition & \multicolumn{10}{c}{Target} \\
\cmidrule(l){2-11}
& 6 & 7 & 12 & 18 & 29 & 36 & 39 & 42 & 46 & 52 \\
\midrule
$C_F^*$ (optimum) & 708.91 & 1322.45 & 1789.55 & 2499.40 & 536.55 & 2710.75 & 2412.34 & 2373.89 & 1194.95 & 2347.89 \\
\midrule
Ariadne Exact        & 708.91 & 1322.45 & 1789.55 & 2499.40 & 536.55 & 2710.75 & 2412.34 & 2373.89 & 1194.95 & 2347.89 \\
Ariadne Greedy       & 708.91 & 1322.45 & 1789.55 & 2499.40 & 536.55 & 2710.75 & 2412.36 & 2373.89 & 1194.95 & 2347.89 \\
Frequency Greedy     & 708.91 & 1322.45 & 1789.55 & 2499.40 & 536.55 & 2710.75 & 2412.36 & 2374.12 & 1194.95 & 2347.89 \\
DKT-set Greedy       & 708.91 & 1322.56 & 1789.66 & 2499.51 & 536.55 & 2710.86 & 2412.51 & 2374.24 & 1194.95 & 2348.00 \\
GPT-5.6 SOL Full     & 708.91 & 1322.49 & 1789.72 & 2499.66 & 536.55 & 2710.87 & 2412.75 & 2374.22 & 1194.95 & 2348.06 \\
DS V4 Pro Full (preview) & 708.91 & 1322.72 & 1789.88 & 2499.63 & 536.55 & 2711.10 & 2412.94 & 2374.44 & 1194.98 & 2348.15 \\
GPT-5.6 SOL Zero     & 708.91 & 1322.77 & 1789.88 & 2499.84 & 536.55 & 2711.22 & 2412.91 & 2374.26 & 1195.18 & 2348.50 \\
DKT-set Exact        & 708.91 & 1322.91 & 1790.00 & 2499.86 & 536.55 & 2711.20 & 2413.01 & 2374.34 & 1195.22 & 2348.34 \\
BKT-set Greedy       & 708.91 & 1322.96 & 1790.06 & 2499.92 & 536.55 & 2711.26 & 2412.87 & 2375.05 & 1194.95 & 2348.40 \\
DS V4 Pro Zero (preview) & 708.91 & 1322.80 & 1790.27 & 2499.93 & 536.55 & 2711.27 & 2413.14 & 2374.63 & 1195.07 & 2348.53 \\
BKT-set Exact        & 708.91 & 1322.91 & 1790.00 & 2499.86 & 536.55 & 2711.21 & 2413.53 & 2375.36 & 1195.22 & 2348.34 \\
Random Frontier      & 708.91 & 1323.09 & 1790.20 & 2500.15 & 536.55 & 2711.46 & 2413.25 & 2375.01 & 1195.07 & 2348.54 \\
Frequency Exact      & 708.91 & 1323.23 & 1789.91 & 2500.28 & 536.55 & 2711.21 & 2413.46 & 2374.67 & 1195.22 & 2348.67 \\
Linear Syllabus      & 708.91 & 1323.33 & 1790.42 & 2500.38 & 536.55 & 2711.72 & 2413.61 & 2374.76 & 1195.22 & 2348.76 \\
\bottomrule
\end{tabular}
\end{table*}

\section{LLM Experimental Matrix and Complete Results}
\label{sec:supp-llm-matrix}

The large language model benchmarking suite evaluates four model-input conditions, accounting for 800 of the 1,890 total planned runs in the experimental run matrix (Table~\ref{tab:supp-run-matrix}). The four conditions represent a full factorial cross of two model families and two information-access levels:

\begin{enumerate}
\item \textbf{GPT-5.6 SOL Zero}: anonymized prerequisite DAG and concept semantics only.
\item \textbf{GPT-5.6 SOL Full}: Zero inputs supplemented with training-split aggregate attempt counts and success rates.
\item \textbf{DeepSeek V4 Pro Zero (preview)}: same as GPT Zero, using the DeepSeek V4 Pro preview API.
\item \textbf{DeepSeek V4 Pro Full (preview)}: same as GPT Full, using the DeepSeek V4 Pro preview API.
\end{enumerate}

\noindent Every condition was executed across all ten prespecified ECS32A target closures with $K=20$ independent logical repetitions per target, instantiated as stateless single-turn API requests without execution tools, external memory, search access, or interactive repair mechanisms.

Tables~\ref{tab:supp-llm-gpt} and \ref{tab:supp-llm-ds} give the complete per-target breakdown. Evaluation costs and normalized regrets are conditional on sequences in $\Sigma_G(s_0)$; invalid runs are excluded from these summaries and audited separately in Section~\ref{sec:supp-llm-validity}.

\begin{table*}[t]
\centering
\small
\caption{Per-target LLM results: GPT-5.6 SOL. Costs are in nominal seconds; regret is normalized by $C_F^*$.}
\label{tab:supp-llm-gpt}
\begin{tabular}{@{}l r r r r r r r r@{}}
\toprule
& \multicolumn{4}{c}{Zero} & \multicolumn{4}{c}{Full} \\
\cmidrule(lr){2-5}\cmidrule(l){6-9}
Target & Valid & Mean cost & Mean regret & Std regret & Valid & Mean cost & Mean regret & Std regret \\
\midrule
6  & 20 & 708.91  & $0$                    & $0$                    & 20 & 708.91  & $0$                    & $0$ \\
7  & 20 & 1322.77 & $2.407\!\times\!10^{-4}$ & $2.568\!\times\!10^{-4}$ & 20 & 1322.49 & $2.709\!\times\!10^{-5}$ & $4.170\!\times\!10^{-5}$ \\
12 & 20 & 1789.88 & $1.893\!\times\!10^{-4}$ & $1.897\!\times\!10^{-4}$ & 20 & 1789.72 & $9.851\!\times\!10^{-5}$ & $1.343\!\times\!10^{-4}$ \\
18 & 20 & 2499.84 & $1.753\!\times\!10^{-4}$ & $1.009\!\times\!10^{-4}$ & 20 & 2499.66 & $1.038\!\times\!10^{-4}$ & $1.113\!\times\!10^{-4}$ \\
29 & 20 & 536.55  & $0$                    & $0$                    & 20 & 536.55  & $0$                    & $0$ \\
36 & 20 & 2711.22 & $1.747\!\times\!10^{-4}$ & $1.045\!\times\!10^{-4}$ & 20 & 2710.87 & $4.478\!\times\!10^{-5}$ & $7.403\!\times\!10^{-5}$ \\
39 & 20 & 2412.91 & $2.325\!\times\!10^{-4}$ & $8.236\!\times\!10^{-5}$ & 20 & 2412.75 & $1.698\!\times\!10^{-4}$ & $5.272\!\times\!10^{-5}$ \\
42 & 20 & 2374.26 & $1.538\!\times\!10^{-4}$ & $1.006\!\times\!10^{-4}$ & 20 & 2374.22 & $1.405\!\times\!10^{-4}$ & $1.176\!\times\!10^{-4}$ \\
46 & 20 & 1195.18 & $1.892\!\times\!10^{-4}$ & $7.949\!\times\!10^{-5}$ & 20 & 1194.95 & $0$                    & $0$ \\
52 & 20 & 2348.50 & $2.631\!\times\!10^{-4}$ & $1.039\!\times\!10^{-4}$ & 20 & 2348.06 & $7.536\!\times\!10^{-5}$ & $9.705\!\times\!10^{-5}$ \\
\bottomrule
\end{tabular}
\end{table*}

\begin{table*}[t]
\centering
\small
\caption{Per-target LLM results: DeepSeek V4 Pro (preview). Conventions as in Table~\ref{tab:supp-llm-gpt}.}
\label{tab:supp-llm-ds}
\begin{tabular}{@{}l r r r r r r r r@{}}
\toprule
& \multicolumn{4}{c}{Zero} & \multicolumn{4}{c}{Full} \\
\cmidrule(lr){2-5}\cmidrule(l){6-9}
Target & Valid & Mean cost & Mean regret & Std regret & Valid & Mean cost & Mean regret & Std regret \\
\midrule
6  & 20 & 708.91  & $0$                    & $0$                    & 20 & 708.91  & $0$                    & $0$ \\
7  & 20 & 1322.80 & $2.642\!\times\!10^{-4}$ & $2.622\!\times\!10^{-4}$ & 20 & 1322.72 & $2.017\!\times\!10^{-4}$ & $2.589\!\times\!10^{-4}$ \\
12 & 19 & 1790.27 & $4.078\!\times\!10^{-4}$ & $1.505\!\times\!10^{-4}$ & 20 & 1789.88 & $1.892\!\times\!10^{-4}$ & $1.766\!\times\!10^{-4}$ \\
18 & 19 & 2499.93 & $2.097\!\times\!10^{-4}$ & $1.385\!\times\!10^{-4}$ & 20 & 2499.63 & $8.927\!\times\!10^{-5}$ & $9.656\!\times\!10^{-5}$ \\
29 & 20 & 536.55  & $0$                    & $0$                    & 20 & 536.55  & $0$                    & $0$ \\
36 & 20 & 2711.27 & $1.909\!\times\!10^{-4}$ & $1.135\!\times\!10^{-4}$ & 20 & 2711.10 & $1.277\!\times\!10^{-4}$ & $1.071\!\times\!10^{-4}$ \\
39 & 17 & 2413.14 & $3.302\!\times\!10^{-4}$ & $1.384\!\times\!10^{-4}$ & 19 & 2412.94 & $2.471\!\times\!10^{-4}$ & $1.156\!\times\!10^{-4}$ \\
42 & 20 & 2374.63 & $3.136\!\times\!10^{-4}$ & $1.743\!\times\!10^{-4}$ & 20 & 2374.44 & $2.315\!\times\!10^{-4}$ & $1.673\!\times\!10^{-4}$ \\
46 & 20 & 1195.07 & $1.002\!\times\!10^{-4}$ & $1.108\!\times\!10^{-4}$ & 20 & 1194.98 & $2.226\!\times\!10^{-5}$ & $6.679\!\times\!10^{-5}$ \\
52 & 20 & 2348.53 & $2.729\!\times\!10^{-4}$ & $1.209\!\times\!10^{-4}$ & 18 & 2348.15 & $1.126\!\times\!10^{-4}$ & $1.290\!\times\!10^{-4}$ \\
\bottomrule
\end{tabular}
\end{table*}

\section{LLM Output Validity and Failure Audit}
\label{sec:supp-llm-validity}

The evaluation pipeline enforces a strict \emph{no-repair policy}: any response that fails syntactic JSON parsing, contains incomplete or extraneous concept coverage, violates a prerequisite edge, or terminates prematurely is flagged as invalid. Such runs receive no evaluation cost or regret and are excluded from conditional performance summaries; they are retained only for computing validity rates and pipeline yields.

Across the 800 planned runs, 799 provider API responses were retrieved, yielding 792 valid sequences in $\Sigma_G(s_0)$, 7 model-invalid responses, and 1 transport-ambiguous execution. Failures partition into two categories: model failures (originating from generation anomalies) and infrastructure failures (originating from the network layer).

\subsection{Failure Taxonomy}

\paragraph{Token budget exhaustion (\texttt{empty\_response}, 4 runs).}
Four DeepSeek runs (V4 Pro Zero: targets 12/18/36; V4 Pro Full: target 18) generated internal reasoning text that exhausted the provider's 32,768-token maximum allocation. The provider terminated with \texttt{finish\_reason=length} before producing a JSON payload, returning empty content. These represent generation-termination failures.

\paragraph{Prerequisite violation (\texttt{prerequisite\_violation}, 3 runs).}
Three DeepSeek runs (V4 Pro Zero: targets 12/42; V4 Pro Full: target 39) produced syntactically complete JSON covering all target-closure nodes, but sequence verification found at least one concept placed before a mandatory ancestor. These represent structural reasoning failures.

\paragraph{Transport ambiguity (\texttt{transport\_ambiguous}, 1 run).}
One DeepSeek V4 Pro Full run (target 39, run 10) encountered an unrecoverable TCP connection reset (\texttt{ConnectionResetError}) during chunked HTTP streaming. Because provider-side completion status could not be verified without breaking execution isolation, this run was conservatively marked invalid. This is an infrastructure fault, not a model competence failure.

\subsection{Validity Summary}

\begin{table}[t]
\centering
\small
\caption{LLM run outcomes and failure modes.}
\label{tab:supp-llm-validity}
\begin{tabular}{@{}l r r r r r r@{}}
\toprule
Condition & Planned & Resp. & Valid & Model fail & Infra fail & Yield \\
\midrule
GPT-5.6 SOL Zero           & 200 & 200 & 200 & 0 & 0 & 100\% \\
GPT-5.6 SOL Full           & 200 & 200 & 200 & 0 & 0 & 100\% \\
DeepSeek V4 Pro Zero (preview) & 200 & 200 & 195 & 5 & 0 & 97.5\% \\
DeepSeek V4 Pro Full (preview) & 200 & 199 & 197 & 2 & 1 & 98.5\% \\
\midrule
Total                      & 800 & 799 & 792 & 7 & 1 & 99.0\% \\
\bottomrule
\end{tabular}
\end{table}

\section{Paired Zero--Full Statistical Analysis}
\label{sec:supp-llm-paired}

We test whether providing aggregate student performance metrics enables LLMs to produce sequences closer to the evaluator-defined optimum than providing prerequisite topology alone.

\subsection{Unit of Analysis}

To prevent larger target closures from dominating inference, the analysis unit is the \emph{target closure} ($N=10$). For each model $M$ and target $t$, the conditional target-level mean normalized regret is
\[
\bar{R}_{M,c}(t)
=
\frac{1}{|K_{t,\mathrm{valid}}|}
\sum_{k\in K_{t,\mathrm{valid}}} R_F(\sigma_{t,k,c}),
\qquad c\in\{\text{Zero},\text{Full}\},
\]
where $K_{t,\mathrm{valid}}$ denotes the valid runs for target $t$. The paired difference is $\Delta R_M(t)=\bar{R}_{M,\text{Full}}(t)-\bar{R}_{M,\text{Zero}}(t)$.

\subsection{Hypothesis Testing and Multiplicity Correction}

For both models, exactly 8 of 10 targets showed strictly negative $\Delta R_M(t)$ (Full improved over Zero), 2 targets were exact ties ($\Delta R_M(t)=0$, on small closures where the valid sequence space is constrained), and 0 targets showed positive differences.

\paragraph{Raw tests.}
With 8 negative differences, 0 positive, and 2 ties, the exact two-sided paired permutation test and the Wilcoxon signed-rank test yield identical raw $p$-values for both models: $p_{\mathrm{raw}}=0.0078125$ (i.e., $1/128$).

\paragraph{Holm--Bonferroni adjustment.}
Two prespecified comparisons (GPT Full--Zero and DeepSeek Full--Zero) require family-wise error rate control. Since $p_{\mathrm{raw},(1)}=p_{\mathrm{raw},(2)}=0.0078125$, the Holm step-down procedure proceeds as follows. The first-ranked hypothesis is tested at $\alpha/(m-1+1)=0.05/2$, yielding $p^{\mathrm{adj}}_{(1)}=0.0078125\times 2=0.015625$. The second-ranked hypothesis has unconstrained multiplier $2-2+1=1$, yielding $0.0078125\times 1=0.0078125$; enforcing the mandatory monotonicity condition $p^{\mathrm{adj}}_{(i)}=\max(p^{\mathrm{adj}}_{(i)},\,p^{\mathrm{adj}}_{(i-1)})$ gives $p^{\mathrm{adj}}_{(2)}=\max(0.0078125,\,0.015625)=0.015625$. Both comparisons thus achieve $p_{\mathrm{Holm}}=0.015625$.

\paragraph{Effect sizes.}
Both models show large standardized effects: Cohen's $d_z=-1.183$ (GPT) and $d_z=-1.298$ (DeepSeek). The matched-pairs rank-biserial correlation is $r_{\mathrm{rb}}=-1.0$ for both models. This boundary value arises because all 8 non-tied targets exhibit strictly negative differences with zero positive counter-examples, so 100\% of concordant pairs align in the negative direction.

\begin{table}[t]
\centering
\small
\caption{Paired target-level differences ($\Delta R=\text{Full}-\text{Zero}$). Nonparametric tests and 95\% bootstrap CIs are computed over $N=10$ paired target means.}
\label{tab:supp-llm-paired}
\begin{tabular}{@{}l r r r r r r@{}}
\toprule
Comparison & Mean $\Delta R$ & 95\% bootstrap CI & $d_z$ & $r_{\mathrm{rb}}$ & Raw $p$ & Holm $p$ \\
\midrule
GPT Full$-$Zero & $-9.59\!\times\!10^{-5}$ & $[-1.44\!\times\!10^{-4},\;-4.92\!\times\!10^{-5}]$ & $-1.183$ & $-1.0$ & $0.0078$ & $0.0156$ \\
DS Full$-$Zero  & $-8.68\!\times\!10^{-5}$ & $[-1.28\!\times\!10^{-4},\;-4.95\!\times\!10^{-5}]$ & $-1.298$ & $-1.0$ & $0.0078$ & $0.0156$ \\
\bottomrule
\end{tabular}
\end{table}

\subsection{Scale Interpretation}

While the relative reduction in mean normalized regret appears large (${\sim}59\%$ for GPT, ${\sim}42\%$ for DeepSeek), this magnitude reflects the extremely small baseline denominator. In absolute terms, the mean target-level cost reduction is approximately $-0.183$ nominal seconds (GPT) and $-0.181$ nominal seconds (DeepSeek). Aggregate student evidence thus consistently guides both models toward evaluator-preferred orderings, but the absolute benefit under a frozen simulator remains modest. Neither the statistical significance nor the relative effect size should be interpreted as evidence of meaningful real-student learning gains.

\section{Solver Comparison}
\label{sec:supp-solver-comparison}

This post-freeze diagnostic compares Full DP, Dijkstra, graph-search A* with the consistent $h_{\mathrm{sum}}$ heuristic, and the instrumented LAO*-derived solver on all ten ECS32A target closures. All four implementations use the same immutable set-based state representation, prerequisite-frontier successor generator, frozen Oracle, and edge cost $T_v/p(v,s)$. Each implementation uses deterministic solver-specific tie-breaking, and all four return the same sequence on every tested closure.

\begin{table*}[t]
    \centering
    \small
    \caption{Expansion counts on the ten ECS32A closures. ``Successor generations'' counts outgoing transitions evaluated, including transitions that reach an already known state; it is not a count of unique states. Full DP and LAO* do not count the terminal state as expanded, whereas A* and Dijkstra do; hence the unit difference. Dijkstra expanded equals A* expanded on every closure and is omitted.}
    \label{tab:supp-solver-expansions}
    \begin{tabular}{@{}rrrrrrr@{}}
        \toprule
        Target & Closure nodes & Ideals & DP expanded & A* expanded & LAO* expanded & Successor gen. \\
        \midrule
        6  & 6  & 7  & 6  & 7  & 6  & 6  \\
        7  & 9  & 20 & 19 & 20 & 19 & 32 \\
        12 & 12 & 23 & 22 & 23 & 22 & 35 \\
        18 & 16 & 28 & 27 & 28 & 27 & 41 \\
        29 & 4  & 5  & 4  & 5  & 4  & 4  \\
        36 & 17 & 29 & 28 & 29 & 28 & 42 \\
        39 & 16 & 47 & 46 & 47 & 46 & 92 \\
        42 & 15 & 35 & 34 & 35 & 34 & 58 \\
        46 & 7  & 9  & 8  & 9  & 8  & 9  \\
        52 & 15 & 26 & 25 & 26 & 25 & 38 \\
        \bottomrule
    \end{tabular}
\end{table*}

\begin{table}[t]
    \centering
    \small
    \caption{Per-target optimal-value agreement across Full DP, Dijkstra, A*, and LAO*-derived search. The gap is the largest absolute difference from the Full-DP value.}
    \label{tab:supp-solver-values}
    \begin{tabular}{@{}rrrr@{}}
        \toprule
        Target & $J^*(s_0)$ & Max.\ abs.\ gap & Max.\ rel.\ gap \\
        \midrule
        6  & 708.911067  & $0$                  & $0$ \\
        7  & 1322.451420 & $2.274\!\times\!10^{-13}$ & $1.719\!\times\!10^{-16}$ \\
        12 & 1789.545179 & $0$                  & $0$ \\
        18 & 2499.402241 & $0$                  & $0$ \\
        29 & 536.547472  & $0$                  & $0$ \\
        36 & 2710.748999 & $4.547\!\times\!10^{-13}$ & $1.678\!\times\!10^{-16}$ \\
        39 & 2412.344931 & $4.547\!\times\!10^{-13}$ & $1.885\!\times\!10^{-16}$ \\
        42 & 2373.890219 & $0$                  & $0$ \\
        46 & 1194.951324 & $0$                  & $0$ \\
        52 & 2347.887157 & $0$                  & $0$ \\
        \bottomrule
    \end{tabular}
\end{table}

\begin{table}[t]
    \centering
    \small
    \caption{Recorded solve-phase wall-clock time in seconds. Each entry is one observation; Oracle construction and heuristic-bound precomputation occur outside the timer. These timings are descriptive diagnostics, not statistical speed estimates.}
    \label{tab:supp-solver-runtime}
    \begin{tabular}{@{}rrrr@{}}
        \toprule
        Target & Full DP & A* & LAO* \\
        \midrule
        6  & 0.0047 & 0.0016 & 0.0020 \\
        7  & 0.0066 & 0.0077 & 0.0086 \\
        12 & 0.0084 & 0.0086 & 0.0099 \\
        18 & 0.0099 & 0.0098 & 0.0118 \\
        29 & 0.0015 & 0.0012 & 0.0018 \\
        36 & 0.0094 & 0.0109 & 0.0124 \\
        39 & 0.0217 & 0.0241 & 0.0258 \\
        42 & 0.0124 & 0.0140 & 0.0163 \\
        46 & 0.0020 & 0.0024 & 0.0028 \\
        52 & 0.0105 & 0.0092 & 0.0105 \\
        \midrule
        Arith.\ mean & 0.0087 & 0.0090 & 0.0102 \\
        \bottomrule
    \end{tabular}
\end{table}

For every target, the runner verifies a cost gap below $10^{-12}\max(1,|J^*_{\mathrm{DP}}|)$, identical returned sequences, and prerequisite-valid complete coverage. The largest observed relative gap is $1.885\times10^{-16}$. A* has zero reopens. Two repeated in-process passes with freshly constructed Oracle objects agree on every non-timing metric; wall-clock time is deliberately excluded from that repeatability check.

On all ECS32A closures, the pruning ratio $R_{\mathrm{prune}}=N_{\mathrm{Dijkstra}}/N_{\mathrm{A*}}$ equals $1.00$---the heuristic performs no pruning. This is expected: the certificate slack (2.5\%--3.2\% residual gap) exceeds the total cost spread ($\sum_v\Delta_v/J^*<0.2\%$) by an order of magnitude, so $f$-values cannot separate optimal from suboptimal states regardless of lattice size. Successor generations, success-probability Oracle calls, and edge-cost evaluations are also identical across the four implementations.

The LAO*-derived wrapper performs repeated Bellman revision passes. For a fully materialized deterministic acyclic state graph, an exact alternative is a single reverse-topological DP pass. The repeated LAO* control structure is therefore structurally avoidable by a solver specialized to this collapsed problem, although revisions remain part of the current incremental LAO* implementation. The single-observation timings above do not isolate their causal runtime contribution.

\section{Sibling-Transfer Trap Family}
\label{sec:supp-trap-family}

This section gives the complete construction of the sibling-transfer trap family used in Section~6 (Q4) and Figure~1 of the main paper.

\subsection{Prerequisite Structure}

For each width parameter \(w\), the concept set is
\[
V=\{r,\,a_1,\ldots,a_w,\,z\},
\]
where \(r\) is a root, \(a_1,\ldots,a_w\) are sibling concepts, and \(z\) is a common sink. The prerequisite edges are \(r\to a_i\) and \(a_i\to z\) for \(i=1,\ldots,w\). The initial state is \(s_0=\varnothing\); \(r\) must be mastered first, after which all \(w\) siblings become simultaneously feasible; \(z\) becomes feasible only after all siblings are mastered. The siblings form a maximum antichain of size \(w\), so the poset width is exactly \(w\) and the total number of concepts is \(w+2\). The experiments use \(w\in\{3,5,7,10,15\}\).

\subsection{Success-Probability Oracle}

Success probabilities follow a logistic pairwise-transfer model:
\[
p(v,s)=\sigma\!\left(\varphi_v+\beta\sum_{u}w_{uv}\,\mathbf{1}[u\in s]\right),
\qquad
\sigma(x)=\frac{1}{1+e^{-x}},
\qquad
\beta=1.5.
\]

\paragraph{Baseline logits.}
\(\varphi_r=1\), \(\varphi_z=0\), and \(\varphi_{a_i}=-0.8+0.12\,i\) for \(i=1,\ldots,w\). Later siblings thus have higher baseline success probabilities.

\paragraph{Transfer weights.}
The only nonzero transfer weights form a chain among the siblings:
\[
w_{a_i,a_{i+1}}=1,\qquad i=1,\ldots,w-1;
\]
all other \(w_{uv}=0\). These are probability-transfer links, \emph{not} prerequisite edges. Mastering \(a_{i-1}\) before attempting \(a_i\) adds \(1.5\) logits to the success probability of \(a_i\).

\paragraph{Attempt cost.}
\(T_v=60\) for every concept, so the collapsed edge cost is \(c(v,s)=60/p(v,s)\) and all cost variation comes from transfer-induced probability differences.

\subsection{Why It Forms a Myopic Trap}

After mastering \(r\), all siblings are feasible. Greedy selects the sibling with the lowest immediate collapsed cost, which is the one with the highest baseline probability---approximately the reverse order \(a_w,a_{w-1},\ldots,a_1\). But the transfer chain runs forward: mastering \(a_{i-1}\) benefits \(a_i\). The reverse order wastes every transfer opportunity, because by the time \(a_{i-1}\) is mastered, \(a_i\) has already been completed. The forward order \(a_1,\ldots,a_w\) incurs higher immediate costs early but accumulates transfer benefits at every subsequent step.

\subsection{State Space and Expansion Counts}

Every subset of siblings, together with the root, forms a valid order ideal:
\[
\{r\}\cup S,\qquad S\subseteq\{a_1,\ldots,a_w\}.
\]
Including \(\varnothing\) and \(V\), the reachable state count is \(|\mathcal S_{\mathrm{reach}}|=2^w+2\). Full DP (which does not count the terminal state as expanded) therefore expands \(2^w+1\) states.

A* with \(h_{\mathrm{sum}}\) expands only the \(w+3\) states on the optimal path: \(\varnothing\to\{r\}\to\{r,a_1\}\to\{r,a_1,a_2\}\to\cdots\to\{r,a_1,\ldots,a_w\}\to V\). This occurs because \(h_{\mathrm{sum}}\) is near-tight on the forward-order path: each sibling's upper-bound probability \(\bar p_{a_i}\) closely matches the probability it actually receives under optimal sequencing (all prior transfer present), so the \(f\)-values of off-path states exceed those on the optimal path and are never expanded. This near-tightness is a property of the construction, not a general guarantee; on instances where \(h_{\mathrm{sum}}\) is loose, A* may expand substantially more states.

Dijkstra (\(h=0\)) provides no pruning and closely tracks the exhaustive DP expansion count, confirming that the reduction in A* expansions is attributable to the heuristic.

\section{Synthetic Random-Poset Variability}
\label{sec:supp-random-poset-variability}

We use width-controlled random posets as a separate supplementary solver diagnostic; this experiment is not part of the primary final-results freeze. Synthetic graphs use \(n\in\{10,15,20,25,30\}\), width parameter \(w\in\{2,3,5\}\), transfer parameter \(\beta\in\{0,0.5,1.0,2.0\}\), and seeds \(42,43,44\). Every \((n,w,\beta)\) cell contains three deterministic seeded realizations. The same seed fixes the graph and node intercepts across values of \(\beta\), so comparisons across \(\beta\) are paired rather than independent. The realized poset width equals the configured \(w\) in all 180 realizations. Tables~\ref{tab:supp-random-poset-a} and \ref{tab:supp-random-poset-b} report the median and interquartile range (IQR, \(Q_3-Q_1\), using linear percentile interpolation) across those three realizations. With only three seeds, this IQR equals one half of the observed range and is descriptive, not an uncertainty interval. \(R_{\mathrm{prune}}\) is Dijkstra expansions divided by tight-bound A* expansions; both counters include the terminal state. For the tight-bound A* run, ``Generated'' counts successor-generation events rather than unique states, and ``Peak heap'' is the maximum number of heap entries, including stale duplicate entries. All 180 realizations completed without timeout.

\begin{table*}[t]
\centering
\small
\setlength{\tabcolsep}{4pt}
\caption{Random-poset variability, \(n=10\) and \(n=15\). Each cell contains seeds 42--44; \(w\) equals the realized width in every run.}
\label{tab:supp-random-poset-a}
\begin{tabular}{@{}rrrr rr rr rr@{}}
\toprule
\(n\) & \(w\) & \(\beta\) & Inst. &
\multicolumn{2}{c}{\(R_{\mathrm{prune}}\)} &
\multicolumn{2}{c}{Generated} &
\multicolumn{2}{c}{Peak heap} \\
\cmidrule(lr){5-6}\cmidrule(lr){7-8}\cmidrule(l){9-10}
& & & & Med. & IQR & Med. & IQR & Med. & IQR \\
\midrule
10 & 2 & 0.0 & 3 & 2.09 & 0.36 & 16 & 1 & 7 & 1 \\
10 & 2 & 0.5 & 3 & 2.09 & 0.36 & 16 & 2 & 7 & 2 \\
10 & 2 & 1.0 & 3 & 1.64 & 0.56 & 19 & 9 & 8 & 1 \\
10 & 2 & 2.0 & 3 & 1.64 & 0.63 & 17 & 8 & 7 & 1 \\
10 & 3 & 0.0 & 3 & 4.09 & 0.40 & 22 & 3 & 13 & 2 \\
10 & 3 & 0.5 & 3 & 4.09 & 0.59 & 22 & 1 & 13 & 1 \\
10 & 3 & 1.0 & 3 & 4.09 & 0.59 & 23 & 1 & 14 & 1 \\
10 & 3 & 2.0 & 3 & 3.36 & 1.71 & 22 & 31 & 13 & 3 \\
10 & 5 & 0.0 & 3 & 9.55 & 0.87 & 33 & 4 & 24 & 1 \\
10 & 5 & 0.5 & 3 & 9.55 & 1.81 & 35 & 2 & 24 & 2 \\
10 & 5 & 1.0 & 3 & 9.55 & 3.43 & 38 & 5 & 25 & 3 \\
10 & 5 & 2.0 & 3 & 9.55 & 2.41 & 33 & 3 & 24 & 3 \\
\midrule
15 & 2 & 0.0 & 3 & 2.19 & 0.13 & 24 & 1 & 10 & 1 \\
15 & 2 & 0.5 & 3 & 1.75 & 0.46 & 34 & 8 & 10 & 2 \\
15 & 2 & 1.0 & 3 & 1.24 & 0.56 & 38 & 13 & 11 & 1 \\
15 & 2 & 2.0 & 3 & 2.19 & 0.13 & 25 & 0 & 11 & 0 \\
15 & 3 & 0.0 & 3 & 4.31 & 0.58 & 31 & 3 & 17 & 2 \\
15 & 3 & 0.5 & 3 & 3.65 & 1.46 & 34 & 44 & 18 & 2 \\
15 & 3 & 1.0 & 3 & 4.05 & 0.22 & 35 & 7 & 21 & 4 \\
15 & 3 & 2.0 & 3 & 3.88 & 0.73 & 35 & 13 & 19 & 3 \\
15 & 5 & 0.0 & 3 & 14.81 & 4.79 & 45 & 19 & 31 & 5 \\
15 & 5 & 0.5 & 3 & 13.82 & 6.28 & 54 & 53 & 38 & 8 \\
15 & 5 & 1.0 & 3 & 8.32 & 6.39 & 102 & 313 & 35 & 6 \\
15 & 5 & 2.0 & 3 & 13.94 & 7.92 & 52 & 172 & 37 & 11 \\
\bottomrule
\end{tabular}
\end{table*}

\begin{table*}[t]
\centering
\small
\setlength{\tabcolsep}{4pt}
\caption{Random-poset variability, \(n=20\), \(25\), and \(30\). Conventions as in Table~\ref{tab:supp-random-poset-a}.}
\label{tab:supp-random-poset-b}
\begin{tabular}{@{}rrrr rr rr rr@{}}
\toprule
\(n\) & \(w\) & \(\beta\) & Inst. &
\multicolumn{2}{c}{\(R_{\mathrm{prune}}\)} &
\multicolumn{2}{c}{Generated} &
\multicolumn{2}{c}{Peak heap} \\
\cmidrule(lr){5-6}\cmidrule(lr){7-8}\cmidrule(l){9-10}
& & & & Med. & IQR & Med. & IQR & Med. & IQR \\
\midrule
20 & 2 & 0.0 & 3 & 1.69 & 0.30 & 40 & 6 & 12 & 2 \\
20 & 2 & 0.5 & 3 & 2.05 & 0.33 & 35 & 4 & 13 & 3 \\
20 & 2 & 1.0 & 3 & 1.91 & 0.38 & 36 & 9 & 11 & 2 \\
20 & 2 & 2.0 & 3 & 1.50 & 0.21 & 45 & 9 & 12 & 4 \\
20 & 3 & 0.0 & 3 & 3.83 & 1.68 & 45 & 54 & 22 & 2 \\
20 & 3 & 0.5 & 3 & 4.26 & 0.69 & 46 & 12 & 23 & 0 \\
20 & 3 & 1.0 & 3 & 4.38 & 1.13 & 39 & 20 & 20 & 3 \\
20 & 3 & 2.0 & 3 & 3.38 & 1.23 & 61 & 31 & 25 & 3 \\
20 & 5 & 0.0 & 3 & 17.52 & 2.93 & 56 & 5 & 37 & 5 \\
20 & 5 & 0.5 & 3 & 9.20 & 3.39 & 129 & 47 & 47 & 7 \\
20 & 5 & 1.0 & 3 & 13.33 & 7.20 & 84 & 364 & 53 & 19 \\
20 & 5 & 2.0 & 3 & 17.52 & 3.28 & 67 & 2 & 47 & 1 \\
\midrule
25 & 2 & 0.0 & 3 & 2.37 & 0.12 & 44 & 3 & 18 & 3 \\
25 & 2 & 0.5 & 3 & 2.46 & 0.34 & 41 & 5 & 16 & 2 \\
25 & 2 & 1.0 & 3 & 2.31 & 0.55 & 45 & 15 & 15 & 6 \\
25 & 2 & 2.0 & 3 & 2.29 & 0.53 & 43 & 10 & 16 & 2 \\
25 & 3 & 0.0 & 3 & 4.58 & 0.23 & 55 & 3 & 26 & 2 \\
25 & 3 & 0.5 & 3 & 4.44 & 1.66 & 53 & 69 & 27 & 2 \\
25 & 3 & 1.0 & 3 & 4.44 & 0.52 & 54 & 9 & 28 & 1 \\
25 & 3 & 2.0 & 3 & 4.62 & 0.26 & 54 & 4 & 30 & 2 \\
25 & 5 & 0.0 & 3 & 13.80 & 1.74 & 86 & 8 & 54 & 3 \\
25 & 5 & 0.5 & 3 & 16.00 & 0.44 & 73 & 3 & 49 & 4 \\
25 & 5 & 1.0 & 3 & 13.80 & 2.24 & 99 & 15 & 59 & 10 \\
25 & 5 & 2.0 & 3 & 16.62 & 1.01 & 77 & 4 & 51 & 3 \\
\midrule
30 & 2 & 0.0 & 3 & 2.34 & 0.29 & 51 & 7 & 21 & 1 \\
30 & 2 & 0.5 & 3 & 1.97 & 0.54 & 63 & 24 & 18 & 4 \\
30 & 2 & 1.0 & 3 & 2.41 & 0.39 & 53 & 13 & 20 & 5 \\
30 & 2 & 2.0 & 3 & 2.48 & 0.07 & 49 & 1 & 20 & 1 \\
30 & 3 & 0.0 & 3 & 3.58 & 0.52 & 93 & 29 & 35 & 3 \\
30 & 3 & 0.5 & 3 & 4.50 & 0.25 & 65 & 15 & 34 & 4 \\
30 & 3 & 1.0 & 3 & 4.81 & 0.32 & 66 & 9 & 33 & 3 \\
30 & 3 & 2.0 & 3 & 4.97 & 1.59 & 63 & 19 & 34 & 1 \\
30 & 5 & 0.0 & 3 & 11.10 & 6.42 & 170 & 82 & 68 & 17 \\
30 & 5 & 0.5 & 3 & 22.61 & 4.09 & 92 & 4 & 62 & 5 \\
30 & 5 & 1.0 & 3 & 21.94 & 3.41 & 104 & 7 & 71 & 7 \\
30 & 5 & 2.0 & 3 & 22.19 & 6.55 & 99 & 25 & 70 & 10 \\
\bottomrule
\end{tabular}
\end{table*}

Across this small seeded design, wider configurations usually increase both successor-generation events and peak heap occupancy because more incomparable concepts create more reachable order ideals. The descriptive spread can also become substantial in wider cells; for example, \((n{=}15,w{=}5,\beta{=}1.0)\) has generation-event median 102 and IQR 313, while \((n{=}30,w{=}5,\beta{=}0)\) has median 170 and IQR 82. A single random-poset instance is therefore not an adequate summary of solver scaling. These three-seed cells provide descriptive algorithmic evidence only and should not be interpreted as population estimates or estimates of student performance.

\section{Junyi Academy Topology and Diagnostic Coverage}
\label{sec:supp-junyi-coverage}

The main text uses the Junyi Academy Math Practicing Log as topology and diagnostic evidence for a wider structural regime. This section documents the coverage limitations of the proxy diagnostic and lists the closures and exercises excluded from it.

\subsection{Data Source and Graph Construction}

The Junyi dataset (PSLC 2015 release) contains 25,925,992 practice-log rows across 837 exercises. After deduplicating two repeated exercise names (\texttt{matrix\_mul\_two}, \texttt{matrix\_app\_fruit\_oil}) by retaining the earliest creation date, 835 exercises remain. The prerequisite graph has 979 raw valid edges; removing 2 self-loops and 1 cycle-inducing edge (\texttt{simplifying\_radicals} $\to$ \texttt{adding\_and\_subtracting\_radicals}) yields a DAG with 978 edges and a transitive reduction of 894 edges. Prerequisite name matching is 100\%; no dangling prerequisites exist.

The full-graph poset width is 301 (computed via maximum bipartite matching on the transitive closure). The graph has 292 sink nodes, each defining a T-sink closure. T-sink closure widths range from 1 to 18 (median 8), and reachable ideal counts range from 2 to above $10^7$ (median 96,608); 68 closures exceed the $10^7$ enumeration guard (all T-sink).

\subsection{Observable Restriction and Coverage}

The proxy diagnostic $m\Delta/J^*_{\mathrm{proxy}}$ reported in the main text is computed on the \emph{observable restriction} of each closure:
\[
V_{t,\mathrm{obs}} = \{v \in V_t : \text{first-attempt success rate exists and is positive}\}.
\]
Of the 835 exercises, 113 have no first-attempt success rate in the practice log, and 12 have a first-attempt success rate of exactly zero (all with fewer than 30 first attempts; see Table~\ref{tab:supp-junyi-zero-rate}). Both groups are excluded from $V_{t,\mathrm{obs}}$, and the exclusion is reflected in two per-closure diagnostics:
\begin{itemize}
\item \textbf{obs\_fraction}: $|V_{t,\mathrm{obs}}|/|V_t|$, the fraction of the closure that is observable. Median 0.85; 58 closures have obs\_fraction $< 0.7$ (Table~\ref{tab:supp-junyi-low-obs}).
\item \textbf{coverage}: within $V_{t,\mathrm{obs}}$, the fraction of concepts for which $\Delta_v$ is not NA. Median 0.78.
\end{itemize}
Of 312 T-sink closures, 241 yield computable proxy values; the remaining 71 are NA. No smoothing, interpolation, or closure deletion is applied.

\subsection{Low-Observability Closures}

Table~\ref{tab:supp-junyi-low-obs} lists the 58 closures with obs\_fraction $< 0.7$. The majority are singleton or very small closures whose target exercise has no practice-log coverage. These closures do not contribute to the proxy diagnostic and are excluded from the median $m\Delta/J^*_{\mathrm{proxy}} = 5.6\%$ reported in the main text.

\begin{table}[t]
\centering
\scriptsize
\caption{Closures with obs\_fraction $< 0.7$. $m$: observable closure size; $m_{\mathrm{total}}$: full closure size.}
\label{tab:supp-junyi-low-obs}
\begin{tabular}{@{}l r r r@{}}
\toprule
Target exercise & $m$ & $m_{\mathrm{total}}$ & obs\_frac \\
\midrule
\texttt{adding\_decimals\_0.5\_new} & 0 & 1 & 0.000 \\
\texttt{adding\_decimals\_1} & 0 & 1 & 0.000 \\
\texttt{adding\_decimals\_2\_new} & 0 & 1 & 0.000 \\
\texttt{angle types} & 0 & 1 & 0.000 \\
\texttt{angle-types} & 0 & 1 & 0.000 \\
\texttt{application\_of\_linear\_equations\_2} & 0 & 1 & 0.000 \\
\texttt{attributs\_of\_shapes} & 1 & 5 & 0.200 \\
\texttt{cc\_and\_cubic\_meter\_conversion} & 0 & 1 & 0.000 \\
\texttt{cell\_bctest\_cap} & 0 & 1 & 0.000 \\
\texttt{classification\_bctest\_cap} & 0 & 1 & 0.000 \\
\texttt{common\_factors\_of\_a\_polynomial\_1} & 0 & 1 & 0.000 \\
\texttt{comparing\_objects} & 0 & 1 & 0.000 \\
\texttt{comparing\_size} & 0 & 2 & 0.000 \\
\texttt{compound\_numbers\_world\_problems} & 0 & 1 & 0.000 \\
\texttt{consistency\_bctest\_cap} & 0 & 1 & 0.000 \\
\texttt{cube\_rectangular\_and\_regular\_tetra\ldots} & 0 & 1 & 0.000 \\
\texttt{cube\_volume\_calculation} & 0 & 1 & 0.000 \\
\texttt{cuboid\_volume\_calculation} & 0 & 1 & 0.000 \\
\texttt{difference\_between\_solid\_and\_plane\ldots} & 0 & 1 & 0.000 \\
\texttt{distributive\_property\_with\_variables} & 0 & 1 & 0.000 \\
\texttt{ecology\_bctest\_cap} & 0 & 1 & 0.000 \\
\texttt{environment\_bctest\_cap} & 0 & 1 & 0.000 \\
\texttt{evolution\_bctest\_cap} & 0 & 1 & 0.000 \\
\texttt{factoring\_difference\_of\_squares\_4} & 0 & 1 & 0.000 \\
\texttt{factoring\_polynomials\_3} & 0 & 1 & 0.000 \\
\texttt{fraction\_and\_per\_swap\_divide} & 0 & 1 & 0.000 \\
\texttt{fractions\_and\_percentage\_of\_exchange} & 0 & 1 & 0.000 \\
\texttt{fractions\_cut\_and\_copy} & 0 & 1 & 0.000 \\
\texttt{g\_and\_kg\_word\_problems} & 0 & 1 & 0.000 \\
\texttt{inheritance\_bctest\_cap} & 0 & 1 & 0.000 \\
\texttt{meters\_and\_centimeters\_conversion} & 0 & 1 & 0.000 \\
\texttt{multiplying\_fractions\_and\_whole\ldots} & 0 & 1 & 0.000 \\
\texttt{number\_sense\_length\_l3} & 0 & 3 & 0.000 \\
\texttt{number\_sense\_weight\_L1} & 0 & 1 & 0.000 \\
\texttt{number\_sense\_weight\_l3} & 0 & 3 & 0.000 \\
\texttt{nutrition\_bctest\_cap} & 0 & 1 & 0.000 \\
\texttt{order\_by\_length} & 0 & 1 & 0.000 \\
\texttt{ordering\_objects} & 0 & 1 & 0.000 \\
\texttt{radius\_angle\_by\_percentage} & 0 & 1 & 0.000 \\
\texttt{recording\_life\_problem\_by\_expressions} & 0 & 1 & 0.000 \\
\texttt{reproduction} & 0 & 1 & 0.000 \\
\texttt{reproduction\_bctest\_cap} & 0 & 1 & 0.000 \\
\texttt{scale\_see\_compare\_see} & 0 & 1 & 0.000 \\
\texttt{seq-arithmetic-sequences-l2} & 0 & 6 & 0.000 \\
\texttt{seq-understaning-arithmetic-seq} & 0 & 1 & 0.000 \\
\texttt{seq-understaning-arithmetic-seq-1} & 0 & 1 & 0.000 \\
\texttt{seq-understaning-arithmetic-seq-2} & 0 & 1 & 0.000 \\
\texttt{seq\_sequence\_exam\_3} & 0 & 5 & 0.000 \\
\texttt{seq\_sequence\_pattern\_2} & 0 & 4 & 0.000 \\
\texttt{short\_divide\_gcf\_and\_lcm} & 0 & 1 & 0.000 \\
\texttt{side\_length\_cube\_and\_relationship\ldots} & 0 & 1 & 0.000 \\
\texttt{simple\_calculation\_of\_composite\_vol\ldots} & 0 & 1 & 0.000 \\
\texttt{simplifying\_rational\_expressions\_1} & 0 & 1 & 0.000 \\
\texttt{simplifying\_rational\_expressions\_2} & 0 & 1 & 0.000 \\
\texttt{solving\_problems\_with\_line\_plots\_1} & 0 & 2 & 0.000 \\
\texttt{tenfold\_and\_hundred\_fold} & 0 & 1 & 0.000 \\
\texttt{transportation\_bctest\_cap} & 0 & 1 & 0.000 \\
\texttt{understand\_percentage} & 0 & 1 & 0.000 \\
\bottomrule
\end{tabular}
\end{table}

\subsection{Zero First-Attempt Success Rate Exercises}

Table~\ref{tab:supp-junyi-zero-rate} lists the 12 exercises whose first-attempt success rate is exactly zero. All have fewer than 30 first attempts, so the zero rate likely reflects sparse data rather than genuine impossibility. These exercises are excluded from $V_{t,\mathrm{obs}}$; no smoothing or interpolation is applied.

\begin{table}[t]
\centering
\small
\caption{Exercises with zero first-attempt success rate.}
\label{tab:supp-junyi-zero-rate}
\begin{tabular}{@{}l r r@{}}
\toprule
Exercise & First att. & Total att. \\
\midrule
\texttt{application\_of\_linear\_equations\_2} & 1 & 41 \\
\texttt{common\_factor\_advance} & 3 & 34 \\
\texttt{completing\_the\_square\_1\_new} & 2 & 7 \\
\texttt{distributive\_property\_with\_variables} & 2 & 5 \\
\texttt{dividing\_decimals\_int\_to\_int\_0.5} & 3 & 8 \\
\texttt{dividing\_decimals\_int\_to\_int\_0.7} & 3 & 7 \\
\texttt{matrix\_app\_fruit\_oil} & 1 & 3 \\
\texttt{matrix\_mul\_two} & 4 & 8 \\
\texttt{multiplying\_fractions\_and\_whole\ldots} & 2 & 2 \\
\texttt{skip\_counting\_by\_10s} & 1 & 1 \\
\texttt{subtraction\_within\_20} & 1 & 8 \\
\texttt{variables\_word\_problems\_2} & 1 & 63 \\
\bottomrule
\end{tabular}
\end{table}

\bibliographystyle{unsrtnat}
\bibliography{references}

\end{document}